\documentclass[twoside]{article}

\usepackage[accepted]{AISTATS2026/aistats2026}
\usepackage{caption}
\usepackage{booktabs}
\usepackage{subcaption}
\usepackage{graphicx} 
\usepackage{amsthm}
\usepackage{xfrac}
\usepackage{multirow}
\usepackage{amssymb}
\usepackage{mathtools}
\usepackage{xcolor}
\usepackage{color,soul}
\usepackage{algpseudocode}
\usepackage{amsmath}
\usepackage{comment}
\usepackage{scalerel}
\usepackage{enumitem}
\usepackage{wrapfig}
\usepackage{nccmath}
\usepackage{array}
\usepackage{enumitem}
\usepackage{longtable}
\usepackage{url}
\usepackage{hyperref}
\usepackage{wrapfig}
\allowdisplaybreaks
\usepackage[ruled]{algorithm2e}
\SetKwInput{Input}{Input}
\SetKwInput{Output}{Output}

\usepackage{array}
\usepackage{paracol}
\globalcounter{algocf}

\usepackage{etoolbox}

\SetAlCapFnt{\footnotesize}
\SetAlCapNameFnt{\footnotesize}

\newtheorem{theorem}{Theorem}
\newtheorem*{theorem*}{Theorem}
\newtheorem{theoremMain}{Theorem}
\newtheorem{corollaryMain}{Corollary}
\newtheorem{corollary}{Corollary}

\newtheorem{lemma}{Lemma}[section]

\newtheorem{assumption}{Assumption}
\newtheorem{definition}{Definition}[section]

\newtheorem*{rep@theorem}{\rep@title}
\newcommand{\newreptheorem}[2]{%
\newenvironment{rep#1}[1]{%
 \def\rep@title{#2 \ref{##1}}%
 \begin{rep@theorem}}%
 {\end{rep@theorem}}}
\newreptheorem{theorem}{Theorem}

\newcommand{\niru}[1]{{\color{red} (Niru) #1}}

\newcommand{\dasha}{\mbox{Byz-DASHA-PAGE}\xspace}

\newcommand{\expect}[1]{\mathop{{}\mathbb{E}}\left[{#1}\right]}
\newcommand{\condexpect}[2]{\mathbb{E}_{#1}\left[{#2}\right]}

\newcommand{\suchthat}{\ensuremath{~\middle|~}}
\newcommand{\knowing}{\suchthat{}}
\newcommand{\card}[1]{\left\lvert{#1}\right\rvert}

\newcommand{\norm}[1]{\left\lVert{#1}\right\rVert}

\newcommand{\indexvar}[3]{\ensuremath{{{#3}^{\ifthenelse{\equal{#1}{}}{}{\left({#1}\right)}}_{#2}}}}
\newcommand{\indexvarNoPar}[3]{\ensuremath{{{#3}^{\ifthenelse{\equal{#1}{}}{}{\left{#1}\right}}_{#2}}}}

\providecommand{\iprod}[2]{\ensuremath{\left\langle #1,\,#2  \right\rangle}}

\providecommand{\norm}[1]{\ensuremath{\left\lVert#1\right\rVert }}

\newcommand{\proba}[2]{\ensuremath{\text{P}\!\left({#1}\ifthenelse{\equal{#2}{}}{}{\knowing{}{#2}}\right)}}

\renewcommand{\paragraph}[1]{\textbf{#1}~}

\newcommand{\mmean}{\overline{m}^t_\mathcal{H}}
\newcommand{\honest}{\mathcal{H}}
\newcommand{\lossHonest}[1]{\nabla \mathcal{L}_\mathcal{H}\left(\theta^{#1}\right)}
\newcommand{\lossWorker}[2]{\nabla \mathcal{L}_{#1}\left(\theta^{#2}\right)}
\newcommand{\mmeanT}[1]{\overline{m}^{#1}_\mathcal{H}}

\newcommand{\lossworker}[2]{\nabla \mathcal{L}_{#1}\left(\theta^{#2}\right)}

\newcommand{\sparHonest}[1]{\overline{\tilde{g}}_\mathcal{H}^{#1}}

\newcommand{\mmmtdrifti}[2]{\mathcal{M}_{#1}^{#2}}

\def\mask{\mathsf{mask}}
\def\R{\mathbb{R}}

\begin{document}

\twocolumn[
    \aistatstitle{Reconciling Communication Compression and Byzantine-Robustness in Distributed Learning}
    \aistatsauthor{Diksha Gupta\And Antonio Honsell\And Chuan Xu\And Nirupam Gupta\And Giovanni Neglia}  
    \aistatsaddress{}
]

\medskip 

\begin{abstract}
    Distributed learning enables scalable model training over decentralized data, but remains hindered by Byzantine faults and high communication costs. While both challenges have been studied extensively in isolation, their interplay has received limited attention. Prior work has shown that naively combining communication compression with Byzantine-robust aggregation can severely weaken resilience to faulty nodes. The current state-of-the-art, Byz-DASHA-PAGE, leverages a momentum-based variance reduction scheme to counteract the negative effect of compression noise on Byzantine robustness. In this work, we introduce RoSDHB, a new algorithm that integrates classical Polyak momentum with a coordinated compression strategy. Theoretically, RoSDHB matches the convergence guarantees of Byz-DASHA-PAGE under the standard $(G,B)$-gradient dissimilarity model, while relying on milder assumptions and requiring less memory and communication per client. Empirically, RoSDHB demonstrates stronger robustness while achieving substantial communication savings compared to Byz-DASHA-PAGE.
\end{abstract}

\section{Introduction}\label{sec:intro}
Distributed learning enables scalable model training over decentralized data by leveraging parallel computation across multiple nodes (workers). In the standard setting, each worker computes local updates on its data and transmits them to a central server, which aggregates these updates to refine a global model. The updated model is then redistributed to the workers, and the process repeats iteratively. This paradigm accelerates training and preserves data locality, but it remains vulnerable to attacks and communication bottlenecks.

A particular challenge arises from malicious or faulty workers, commonly referred to as \textit{Byzantine} workers in distributed computing literature~\cite{lamport82}. Such workers can arbitrarily corrupt their updates, severely degrading the convergence and accuracy of the global model. This has motivated a large body of work on Byzantine-robust distributed learning~\cite{blanchard_2017_adversaries, brute_bulyan, yin2018byzantine, li2021ditto, pillutla2022robust, karimireddy2022byzantinerobust, allouah2023fixing}, which develops robust aggregation rules to mitigate the impact of corrupted updates during server-side aggregation. These defenses are especially critical in open or partially trusted environments, such as \textit{federated learning}~\cite{mcmahan17fedavg}, where the non-iid nature of data across devices makes distinguishing Byzantine updates from legitimate but heterogeneous ones particularly challenging.

Another major line of work focuses on improving the communication efficiency of distributed learning. Transmitting full model parameters or gradients from every worker to the server at each round is often impractical under bandwidth constraints. To address this, communication compression schemes such as gradient sparsification~\cite{stich2018sparsified} and quantization~\cite{alistarh2017qsgd} have been widely studied (e.g.,~\cite{chen2021communication, hamer2020fedboost, reisizadeh2020fedpaq, richtarik2021ef21, oh2024communication}). While these methods effectively reduce communication costs, they have mostly been analyzed in isolation from adversarial robustness. The interaction between communication compression and Byzantine resilience remains relatively underexplored, with only a few notable attempts~\cite{bernstein2018signsgd, sattler2019robust, gorbunov2023variance, rammal2024communication}.

The current state-of-the-art (SOTA) approach for combining Byzantine-robustness with communication compression in distributed learning is \dasha\cite{rammal2024communication}. This algorithm considers a robust variant of distributed stochastic gradient descent (SGD) with \textit{unbiased compression}, and addresses the adverse effect of compression on Byzantine-robustness through \textit{momentum variance reduction}—a technique originally proposed in~\cite{cutkosky2019momentum} to achieve optimal convergence rates under noisy gradients. 
This mechanism is central to the significant improvements of \dasha over earlier methods such as Byz-VR-MARINA~\cite{gorbunov2023variance}, but it comes at the cost of maintaining additional momentum states \emph{before compression}, i.e., on the worker side, which leads to substantial per-worker memory usage.
Moreover, the convergence of \dasha relies on a special  property of loss functions, namely bounded global Hessian variance~\cite{szlendak2021permutation} (see Appendix~\ref{app:ghv}).

We address the aforementioned limitations of \dasha by proposing a new algorithm, named \textbf{Ro}bust \textbf{S}parsified \textbf{D}istributed \textbf{H}eavy-\textbf{B}all (\textbf{RoSDHB}). 
The design of RoSDHB is deliberately simpler than \dasha. Instead of the more involved \textit{momentum variance reduction} scheme, RoSDHB employs the classical \textit{Polyak’s momentum} (a.k.a.~heavy-ball method). 
Crucially, RoSDHB applies heavy-ball momentum \emph{after} compression, i.e., once the compressed gradients have been received by the server. This design shift moves the momentum operation to the server and halves the memory requirements at each worker, while still effectively mitigating the noise introduced by compression.
A further innovation of RoSDHB is the use of a \textit{global mask}, i.e., a shared sparsity pattern applied across all workers. This design 
reduces the uplink communication footprint, as workers transmit only the jointly selected coordinates and not the corresponding indices. 
While offering both memory and communication improvements, 
the convergence guarantee of RoSDHB is comparable to that of \dasha in the gradient descent setting, without the additional assumption of bounded global Hessian variance. 

In particular, we establish tight convergence guarantees RoSDHB on smooth nonconvex loss functions under the standard $(G,B)$-gradient dissimilarity~\cite{karimireddy2020scaffold} and Lipschitz smoothness assumptions. Unlike prior momentum-based results for Byzantine-robust distributed learning~\cite{Karimireddy2021, farhadkhani2022byzantine}, our analysis addresses sparsification noise whose variance grows with the gradient norm of the average loss, rather than being uniformly bounded. To handle this more challenging regime, we develop a new Lyapunov function, which could be of independent interest to the distributed optimization community. 

Although RoSDHB achieves convergence rates comparable to \dasha—under weaker and more standard assumptions—our experimental results reveal a clear practical advantage, in both full gradient and stochastic gradient scenarios. In particular, RoSDHB demonstrates stronger robustness to Byzantine attackers and achieves significant communication savings across benchmark image classification tasks.

{\bf In summary, our contributions are three-fold:} (i) We propose a new distributed learning algorithm RoSDHB that provides Byzantine-robustness while enforcing communication compression. Compared to prior work, RoSDHB incurs reduced memory and communication costs for the workers. (ii) 
We show (for the first time) that simple heavy-ball momentum, applied after sparsification and in combination with a global mask, yields tight Byzantine-robustness under standard assumptions. (iii) Our experiments show that RoSDHB consistently outperforms the SOTA method in practice, achieving superior robustness to Byzantine attackers with significant communication savings.

The rest of the paper is organized as follows. Section~\ref{sec:background} reviews background material on Byzantine-robustness and gradient sparsification. Section~\ref{sec:algo} introduces our algorithm RoSDHB, presents its theoretical analysis, and includes a formal comparison with SOTA methods. 
Section~\ref{sec:experiments} reports empirical results on benchmark machine learning tasks. Finally, Section~\ref{sec:conc_rw}  
summarizes our contributions, and outlines promising directions for future research.

\section{Problem Setting \& Background}\label{sec:background}
        




In this section, we present the problem setup and the assumptions used throughout the paper. We use $\norm{\cdot}$ to denote the Euclidean norm.

\paragraph{Distributed learning.} 
We consider a server-based distributed learning setting with one central server and $n$ workers $[n]\coloneq\{1, 2, \ldots, n\}$.  Each worker $i$ holds a local dataset $D_i = \{z_i^1, z_i^2, \ldots, z_i^{m}\}$, drawn from a possibly different distribution, as is often the case in \textit{federated learning} settings.
For each data point $z \in \mathcal{Z}$, a model with parameter vector $\theta \in \mathbb{R}^d$ incurs a loss $\ell(\theta, z)$, where $\ell$ is real-valued and differentiable. The local empirical loss at worker $i$ is $\mathcal{L}_i(\theta) = \tfrac{1}{m} \sum_{j=1}^{m} \ell(\theta, z_i^j)$, and the global objective is to minimize the average loss $\mathcal{L}(\theta) = \tfrac{1}{n} \sum_{i=1}^n \mathcal{L}_i(\theta)$.\footnote{
The analysis remains valid both when workers hold datasets of different sizes and when the global loss is defined as a weighted combination of the local losses.
} We allow $\mathcal{L}(\theta)$ to be nonconvex, in which case the goal is to find a stationary point (defined formally later).
To this purpose, the server can employ distributed gradient descent (DGD). At each iteration $t \in {1, \ldots, T}$, the server broadcasts the current model $\theta^{t-1}$ to all workers. Each worker $i$ computes its local gradient $g_i^t = \nabla \mathcal{L}_i(\theta^{t-1})$ and sends it to the server. The server aggregates them as $g^t = \tfrac{1}{n} \sum_{i=1}^n g_i^t$ and updates the model by $\theta^t = \theta^{t-1} - \gamma g^t$, where $\gamma > 0$ is the learning rate. The initial model $\theta^0$ is chosen arbitrarily.




\paragraph{Communication Compression.} 
To reduce the communication cost of DGD, several compression schemes have been proposed~\cite{beznosikov2023biased}. We focus on \textit{RandK sparsification}~\cite{stich2018sparsified}, where at each iteration $t$ worker $i$ sends only $k$ ($k \leq d$) randomly selected coordinates of its local gradient $g_i^t$.
We define the \textit{compression parameter} as $\alpha \coloneqq d/k$. 
Formally, each worker samples a binary mask $\mask_i(k) \in \{0,1\}^d$ with exactly $k$ entries equal to $1$, indicating the selected coordinates. The compressed message is the pair $(\mask_i(k), \mathcal{C}_k(g_i^t))$, where $\mathcal{C}_k(g_i^t) = ([g_i^t]_{\ell_1}, \ldots, [g_i^t]_{\ell_k})$ contains the gradient values at the selected indices $\ell_1 < \cdots < \ell_k$ satisfying $[\mask_i(k)]_{\ell_j} = 1$.
Using this information, the server reconstructs
$\tilde{g}_i^t  = \alpha (g^t_i \odot \mask_i(k)) \in \R^d$, where $\odot$ denotes the element-wise (Hadamard) product. The server then averages $\overline{\tilde{g}}^t = \tfrac{1}{n} \sum_{i=1}^n \tilde{g}_i^t$ and updates the model as $\theta^t = \theta^{t-1} - \gamma \overline{\tilde{g}}^t$.
RandK is an unbiased compression operator with bounded variance~\cite[Def.~1]{beznosikov2023biased}, since
$$\expect{\tilde{g}_i^t} = g^t_i ~ \text{ and } ~ \expect{\norm{\tilde{g}_i^t - g^t_i}^2} \leq (\alpha - 1) \norm{g^t_i}^2,$$
where the expectation is over the randomness of $\mask_i(k)$.

\paragraph{Byzantine-robustness.} 
We consider a scenario wherein an adversary corrupts a set of at most $f$ workers out of total $n$, whose identity is fixed but a priori unknown to the server. We refer to these corrupted workers as {\em Byzantine workers}, as per the terminology used in distributed computing~\cite{lamport82}, and the remaining uncorrupted workers as the {\em honest workers}. We denote by $\mathcal{H} \subseteq [n]$  the set of honest workers ($|\mathcal{H}| = n-f$).
Byzantine workers can deviate arbitrarily from the prescribed algorithm. We consider the worst-case scenario where they may collude seamlessly, know the learning algorithm used by the server, and observe all messages exchanged between the server and the honest workers throughout training.
In the presence of Byzantine workers, the server’s objective is to find a stationary point of the average loss over the honest workers, $\mathcal{L}_{\mathcal{H}}(\theta) \coloneqq \frac{1}{\card{\mathcal{H}}} \sum_{i \in \mathcal{H}}  \mathcal{L}_i\!\left( \theta \right)$.
An algorithm that, in the presence of at most $f$ Byzantine workers, guarantees convergence to an $\epsilon$-stationary point of $\mathcal{L}_{\mathcal{H}}$ is said to be $(f,\epsilon)$-resilient. Formally~\cite{allouah2023fixing}: 

\begin{definition}[$(f, \epsilon)$-Resilience]
A distributed learning algorithm is said to be $(f, \epsilon)$-resilient if, in the presence of at most $f$ Byzantine workers, it outputs a parameter $\hat{\theta} \in \mathbb{R}^d$ such that $\norm{\nabla \mathcal{L}_{\mathcal{H}}(\hat{\theta})}^2 \leq \epsilon$.
\end{definition}

The convergence of distributed first-order methods typically relies on standard conditions concerning the regularity of the global loss and the heterogeneity of local losses. In this paper, we adopt two broad assumptions that capture these aspects. These assumptions are imposed only on the honest workers, while Byzantine workers are left unconstrained and may act arbitrarily to disrupt training. We first introduce the following smoothness assumption on the honest workers’ global loss, which is standard even in centralized first-order machine learning methods~\cite{bottou2018optimization}.

\begin{assumption}{(Smoothness)}\label{assumption:smoothness}
    The average loss function $\mathcal{L}_\mathcal{H}: \mathbb{R}^d \rightarrow \mathbb{R}$ is $L$-Lipschitz smooth, i.e., there exists $L \geq 0$ such that for all $\theta, \theta' \in \mathbb{R}^d$,
    $$\norm{\nabla\mathcal{L}_\mathcal{H}(\theta) - \nabla\mathcal{L}_\mathcal{H}(\theta')} \leq L \norm{\theta - \theta'} .$$
\end{assumption}

To model data heterogeneity across honest workers, we rely on the standard notion of $(G,B)$-gradient dissimilarity~\cite{karimireddy2020scaffold, allouah2023robust}.
\begin{assumption}{($(G,B)$-gradient dissimilarity)}\label{assumption:gradientGB}
    There exist two non-negative constants $G$ and $B$, such that, for all $\theta \in \mathbb{R}^d$, we have:
        $$\textstyle \frac{1}{|\mathcal{H}|}\sum\limits_{i\in \mathcal{H}} \norm{\lossworker{i}{}- \lossworker{\mathcal{H}}{}}^2 \leq G^2 + B^2\norm{\lossworker{\mathcal{H}}{}}^2.$$
\end{assumption}


DGD with simple averaging as the server-side aggregation rule, $g^t = \tfrac{1}{n} \sum_{i=1}^n g_i^t$, does not achieve $(f,\epsilon)$-resilience~\cite{krum}.  Resilience requires limiting the influence of Byzantine workers, which can be achieved, for example, by replacing averaging with a \emph{robust aggregation} rule~\cite{allouah2023fixing}.
\begin{definition}[$(f,\kappa)$-Robustness]
\label{def:robust_aggregation}
An aggregation rule $F : (\mathbb{R}^d)^n \to \mathbb{R}^d$ is said to be $(f,\kappa)$-robust if, for any set of $n$ vectors $\{x_1,\ldots,x_n\} \subseteq \mathbb{R}^d$ and any subset $S \subseteq [n]$ of size $n-f$, it holds that
\[
    \norm{F(x_1,\ldots,x_n) - \overline{x}_S}^2 \leq \frac{\kappa}{|S|}\sum_{i \in S}\norm{x_i - \overline{x}_S}^2,
\]
where $\overline{x}_S \coloneqq \tfrac{1}{|S|} \sum_{i \in S} x_i$.
\end{definition}

The above definition encompasses several robust aggregation rules, including coordinate-wise trimmed mean (CWTM), and geometric median (GeoMed)~\cite[Chapter 4]{guerraoui2024robust}, which are $(f, \kappa)$-robust for different values of $\kappa$, where $\kappa$ typically depends on $f$ and $n$, and increases with $f$.

We call \textit{robust DGD} a variant of DGD where averaging is replaced by an $(f,\kappa)$-robust aggregation rule. Specifically, at each iteration $t$, the server updates the parameter vector as $\theta^t = \theta^{t-1} - \gamma R^t$, where $R^t \coloneqq F(g^t_1, \ldots, g^t_n)$. Under Assumptions~\ref{assumption:smoothness} and~\ref{assumption:gradientGB}, it has been shown~\cite{allouah2023robust} that if $\kappa = \frac{f}{n-2f}$, robust DGD achieves $(f,\epsilon)$-convergence with
$\epsilon \in O\!\left( \tfrac{f}{n - (2 + B^2) f}\right)$,
which matches the lower bound on $\epsilon$ established in~\cite{allouah2023fixing}. When $n \geq (2+\nu)f$ for some $\nu > 0$, robust aggregation rules such as CWTM and GeoMed (when combined with nearest-neighbor mixing, NNM) achieve $\kappa \in \Theta\!\left(\tfrac{f}{n-2f}\right)$~\cite{allouah2023fixing}.


\paragraph{Integrating Robustness and Communication Compression.}
Achieving both communication compression, to alleviate bandwidth constraints, and robustness, to mitigate the impact of Byzantine workers, is highly desirable but nontrivial. Simply adding compression schemes such as RandK to robust DGD is insufficient; more sophisticated algorithms are required.
Some works attempt this integration but under restrictive assumptions: SignSGD with majority vote~\cite{bernstein2018signsgd} requires honest gradients to follow identical unimodal distributions, while norm filtering with error feedback under biased compression~\cite{ghosh2021communication} assumes subexponential gradients. The current state-of-the-art method addressing this joint goal is \dasha~\cite{rammal2024communication}. Its key ingredient is \emph{momentum variance reduction}, a scheme originally proposed in~\cite{cutkosky2019momentum} to attain optimal convergence rates for SGD on smooth nonconvex functions.
\dasha, however, comes with notable memory costs. Each worker must store four vectors of dimension $d$ (the number of model parameters): the current parameter vector, the current gradient, the gradient from the previous communication round, and an additional pseudogradient. On the server side, the memory requirement is $2(n+1)d$, as it must store each worker’s pseudogradient and current update, along with the global model and the robust aggregate. Furthermore, under RandK sparsification, each worker uses local masks and therefore transmits $k$ values together with the $k$ corresponding indices (or, equivalently, $d$ bits). Finally, the resilience guarantees of \dasha rely not only on Assumptions~\ref{assumption:smoothness} and~\ref{assumption:gradientGB}, but also on the additional condition of \emph{bounded global Hessian variance}~\cite{szlendak2021permutation} (reported as Assumption~\ref{assumption:ghv} in Appendix~\ref{app:ghv}). Unless each worker local loss  is $L$-smooth, it is easy to construct an example when the global loss is smooth, but bounded global Hessian variance is not satisfied (see Appendix~\ref{app:ghv}). 

As we show in the following, our algorithm RoSDHB overcomes these limitations of \dasha.

\section{Our Algorithm: Robust Sparsified Distributed Heavy-Ball (RoSDHB)}\label{sec:algo}

\noindent In this section, we present our algorithm Robust Sparsified Distributed Heavy-Ball (RoSDHB). The pseudocode is provided in Algorithm~\ref{alg:rghb}, while Section~\ref{sec:desc_RoSDHB} details its key components and complexity. Section~\ref{sec:analysis_RoSDHB} then presents the convergence analysis of RoSDHB and compares it with state-of-the-art results, including \dasha.



\begin{algorithm*}
    \caption{\textbf{Ro}bust \textbf{S}parsified \textbf{D}istributed \textbf{H}eavy-\textbf{B}all (\textbf{RoSDHB})}
    \textbf{Input:} Initial model $\theta^0 \in \R^d$ (chosen arbitrarily), total iterations $T \ge 1$, learning rate $\gamma$, robust aggregator $F$, momentum coefficient $\beta \in [0, 1)$, sparsification parameter $k$, 
    and, for each honest worker $w_i$: $m_i^0 = 0$.\\
  
    \vspace{5pt}
    \textbf{For} $t = 1$ to $T$:
    \begin{algorithmic}[1]
        \State \textbf{Server} generates $\mask(k) \in \{0, 1\}^d$ with $k$ randomly chosen elements set to $1$ and rest to $0$.  
        \State \textbf{Server} broadcasts current model $\theta^{t-1}$ and $\mask(k)$ to all the workers. \label{step:server}
        \State For each \textbf{honest worker} $i$ do:
        \begin{enumerate}[wide,labelindent=0pt, label=\alph*.]
            \item Compute local gradient: $g_i^t = \lossWorker{i}{t-1}$ .
            \item Send to the server: $\mathcal{C}_k\left( g_i^t\right) = \left( [g_i^t]_{\ell_1}, \ldots,[g_i^t]_{\ell_k} \right)$ where $\ell_1 < \ldots < \ell_k$, $[\mask(k)]_{\ell_{j}} = 1, \forall j \in [k]$ .
        \end{enumerate}

        \noindent \texttt{$\slash\slash$ A Byzantine worker $j$ can send arbitrary $k$ values in $\mathcal{C}_k\left( g_j^t\right)$. }
        
        \State \textbf{Server} computes for each worker $i$, : $\tilde{g}_i^t = \frac{d}{k} (g_i^t \odot \mask(k))$ using $\mathcal{C}_k\left( g_i^t\right)$. 
        \State \textbf{Server} computes momentum for each worker $i$: $m_i^t = \beta m_i^{t-1} + (1-\beta)\tilde{g}_i^t$ .

        \State \textbf{Server} aggregates the momenta:
        $R^t = F(m_1^t, \ldots ,m_n^t)$ .

        \State \textbf{Server} updates the model: $\theta^{t} = \theta^{t-1} - \gamma R^t $ ,   
        
    \end{algorithmic} 
    \textbf{\textbf{Server} returns} $\hat{\theta}$, chosen uniformly at random from $\{\theta^0, \ldots, \theta^{T-1}\}$.
    \label{alg:rghb}

\end{algorithm*}

\subsection{Description of RoSDHB}\label{sec:desc_RoSDHB}

In each iteration $t \in \{1, \ldots, T\}$, with $T \ge 1$ denoting the total number of iterations, the server broadcasts to all workers (i) the current model parameter $\theta^{t-1}$ and (ii) a RandK sparsification mask $\mask(k) \in \{0,1\}^d$, obtained by setting $k$ randomly chosen entries to $1$ and the rest to $0$. The initial model $\theta^0$ is chosen arbitrarily by the server. Upon receiving $\theta^{t-1}$ and the mask, each honest worker $i$ computes its local gradient $g_i^t = \nabla \mathcal{L}_i(\theta^{t-1})$ and sends back the $k$ coordinates corresponding to the nonzero entries of the mask, i.e., $\mathcal{C}_k(g_i^t) = ([g_i^t]_{\ell_1}, \ldots, [g_i^t]_{\ell_k})$ with $\ell_1 < \cdots < \ell_k$ and $[\mask(k)]_{\ell_j} = 1$ for all $j \in [k]$. In contrast, a Byzantine worker $j$ may transmit any arbitrary set of $k$ values in $\mathcal{C}_k(g_j^t)$.
From the received messages, the server reconstructs an unbiased estimator of each local gradient as
$$\textstyle \tilde{g}^t_i = \frac{d}{k} \left( g_i^t \odot \mask(k)\right).$$
The server then updates the momentum vector $m_i^t$ of each worker as
$$\textstyle m_i^t = \beta m_i^{t-1} + (1-\beta)\tilde{g}_i^t,$$
with momentum coefficient $\beta \in [0,1)$ and initialization $m_i^0 = 0$. Next, the server applies the robust aggregation rule $F$ to obtain $R^t = F(m_1^t, \ldots, m_n^t)$, and updates the model as
$$\theta^{t} = \theta^{t-1} - \gamma R^t.$$

After $T$ iterations, the server outputs $\hat{\theta}$, chosen uniformly at random from ${\theta^0, \ldots, \theta^{T-1}}$. The use of momentum vectors in the update rule mirrors Polyak’s momentum method, also known as the heavy-ball method, which motivates the name of our algorithm.

 \paragraph{RoSDHB complexity.} RoSDHB is lightweight in terms of computation, memory, and communication. Each honest worker computes one local gradient per iteration, incurring the same computational cost as in standard DGD. On the memory side, each worker maintains only two $d$-dimensional vectors—the current parameter vector and gradient—in contrast to \dasha, which requires four such vectors. The server stores the global model together with one momentum vector per worker, for a total memory footprint of $2(n+1)d$, identical to \dasha. Communication from workers to the server is slightly reduced: with a global mask, each worker transmits only $k$ gradient values per iteration without sending the corresponding indices, whereas \dasha with RandK requires transmitting both $k$ values and their indices (or equivalently $d$ bits). This saving comes at the cost of distributing the mask from the server to the workers. However, 
 this overhead of transmitting masks can be made negligible if they are generated by a random number generator shared between the server and the workers, initialized with a common seed, so that only the initial seed must be communicated. Overall, RoSDHB incurs strictly lower memory and worker-to-server communication overhead than \dasha, while preserving the same computational complexity.

\subsection{Analysis of RoSDHB}\label{sec:analysis_RoSDHB}





We prove the following theorem on the convergence of RoSDHB in Appendix \ref{app:glorhb}.

\begin{theoremMain}\label{thm:rosdhb}
    Under Assumptions~\ref{assumption:smoothness} and~\ref{assumption:gradientGB}, Algorithm~\ref{alg:rghb} with an $(f,\kappa)$-robust aggregation rule $F$ such that $\kappa B^2 \leq \tfrac{1}{7}$, a learning rate $\gamma \leq \tfrac{k}{d c L}$ (with $c = 23200$), and a momentum coefficient $\beta = \sqrt{1 - 24\gamma L}$, satisfies the following guarantee
    \begin{align*}
        \expect{\norm{\nabla \mathcal{L}_\mathcal{H}(\hat{\theta})}^2} \leq \frac{45 \left(\mathcal{L}_\mathcal{H}(\theta^0) - \mathcal{L}_\mathcal{H}^*\right)}{\gamma T ( 1- \kappa B^2)} + \frac{216 \kappa \, G^2}{1 - \kappa B^2},
    \end{align*}
    where $\mathcal{L}_\mathcal{H}^* = \min_{\theta \in \R^d} \mathcal{L}_\mathcal{H}(\theta)$.
\end{theoremMain}
\begin{proof}[Proof sketch]
    The proof relies on the following novel Lyapunov function:
$$V^t \coloneqq \mathbb{E}\!\left[2\mathcal{L}_\mathcal{H}(\theta^{t-1}) + \tfrac{1}{8L}\lVert\delta^t\rVert^2 + \tfrac{\kappa}{4L}\Upsilon^{t-1}_\mathcal{H}\right].$$ 
Here $\delta^t \coloneqq \overline{m}^t - \nabla \mathcal{L}_\mathcal{H}(\theta^{t-1})$, with $\overline{m}^t \coloneq \tfrac{1}{|\mathcal{H}|} \sum_{i \in \mathcal{H}} m_i^t$, measures the deviation of the averaged momentum from the actual gradient. The term $\Upsilon^{t-1}_\mathcal{H} \coloneqq \tfrac{1}{|\mathcal{H}|} \sum_{i \in \mathcal{H}} \lVert m_i^{t-1} - \overline{m}_{\mathcal{H}}^{t-1}\rVert^2$ instead captures the drift of individual client momenta from their average. 
\end{proof}

Recall that $\alpha = d/k$ denotes the RandK compression ratio.
We obtain the following corollary.
\begin{corollaryMain}\label{co:rosdhb}
    In the same conditions of Theorem~\ref{thm:rosdhb},  let $\gamma = \tfrac{1}{\alpha c L}$, then
        \begin{align}            \label{eq:corollary_bound}
         \textstyle   \expect{\norm{\nabla \mathcal{L}_\mathcal{H}(\hat{\theta})}^2} \leq \mathcal{O}\! \left( \frac{\alpha}{T ( 1- \kappa B^2)} + \frac{\kappa G^2}{1 - \kappa B^2} \right).
        \end{align}
\end{corollaryMain}

\paragraph{On the tightness of our result.} 
When $\alpha = 1$ (i.e., no compression), Corollary~\ref{co:rosdhb} recovers the same bound as the SOTA result without compression~\cite{allouah2023robust}. Hence, the reasoning in Sec.~5.1 of~\cite{allouah2023robust} directly implies that the convergence bound in Corollary~\ref{co:rosdhb} is optimal under the best possible robust aggregation rule, namely one with $\kappa = f/(n-2f)$. 

\paragraph{Interplay between compression and robustness.}
The corollary also highlights the interplay between compression and robustness. When $B=0$, it separates the two effects in a manner analogous to what was observed in~\cite{allouah2023fixing} under $(G,0)$-heterogeneity: one term reflects the impact of compression, and the other the impact of Byzantine robustness. By contrast, when $B>0$, the first term in the bound shows that compression and robustness no longer decouple but instead amplify each other’s effect: stronger compression (higher $\alpha$) increases the influence of Byzantine robustness, and a larger fraction of Byzantine workers exacerbates the impact of compression.

\noindent \textbf{Comparison with \dasha.} The SOTA method addressing both robustness and compression is \dasha~\cite{rammal2024communication}.
Our algorithm’s convergence guarantee in Theorem~\ref{thm:rosdhb} matches that of \dasha, but under fewer assumptions. In particular, our analysis requires only Lipschitz smoothness of the honest-average loss, whereas \dasha further assumes \emph{bounded global Hessian variance} for the honest workers’ losses.
For a fair comparison, we adapt the guarantee reported in~\cite{rammal2024communication} to the gradient–descent setting (with $n-f \ge 2$) and evaluate it in the most favorable regime for \dasha—namely, setting the Hessian-variance term to zero and omitting some terms. Let $\tilde{\theta}$ denote its output after $T$ iterations. The resulting \emph{best-case} upper bound, derived in Appendix~\ref{app:dasha}, is:
\begin{align*}
\textstyle
    \expect{\norm{\nabla \mathcal{L}_\mathcal{H}\left(\tilde{\theta}\right)}^2} \lesssim \mathcal{O} \left(\frac{1+4(\alpha-1) \left(\frac{1}{\sqrt{n}}+\sqrt{\kappa}\right)}{T(1-2\kappa B^2) } + \frac{\kappa G^2}{1-2\kappa B^2} \right).
\end{align*}
Comparing the above rate with the bound for RoSDHB in Corollary~\ref{co:rosdhb}, two differences stand out. First, \dasha has a worse dependence on the robustness parameter $\kappa$: both terms are divided by $(1-2\kappa B^2)$ instead of $(1-\kappa B^2)$ as in RoSDHB, and the vanishing term carries an additional $\sqrt{\kappa}$ factor. Since $\kappa \ge f/(n-2f)$, this penalty grows as the Byzantine fraction increases. Second, the impact of the compression ratio $\alpha$ is milder for \dasha when $f/n$ is small, because $\alpha$ appears only through $(\alpha-1)(1/\sqrt{n}+\sqrt{\kappa})$, whereas for RoSDHB the $1/T$ term scales linearly with $\alpha$. Consequently, given that $\kappa \in \mathcal{O}(f/n)$ for state-of-the-art robust aggregation rules~\cite{allouah2023fixing}, \dasha incurs a lighter compression penalty in the small-$f$, large-$n$ regime.
\section{Empirical Evaluation}\label{sec:experiments}


\subsection{Experimental setup}
\paragraph{Datasets and Models.} 
 We evaluate our method on two standard image classification benchmarks: MNIST and CIFAR-10. The MNIST dataset~\cite{mnist} consists of grayscale images of handwritten digits, with 60,000 training and 10,000 test samples across 10 digit classes. CIFAR-10~\cite{cifar} contains 32×32 RGB images, comprising 50,000 training and 10,000 test samples over 10 object categories. For MNIST, we use a lightweight neural network with one convolutional layer followed by a fully connected layer, totaling about 12K parameters. For CIFAR-10, we employ a deeper ResNet-18 model~\cite{he2016deep}, with about 11M parameters.


\paragraph{Distributed system setup.} We consider a server-based distributed learning scenario with $10$ honest workers and $f$ Byzantine attackers. To model real-world data heterogeneity, we distribute samples with the same label across clients following a symmetric Dirichlet distribution with parameter $w=5$~\cite{wang2020federated}. The server applies the coordinate-wise trimmed mean (CWTM) robust aggregation method~\cite{yin2018byzantine}, while clients compress their updates using the RandK sparsifier~\cite{beznosikov2023biased}, using a sampling ratio $1/\alpha = k/d$ that retains only $k$ out of $d$ coordinates.
We evaluate resilience against two attacks: the Fall of Empires (FOE)~\cite{xie2020fall} and A Little Is Enough (ALIE)~\cite{baruch2019alittle}.
Detailed descriptions of CWTM, FOE and ALIE are provided in Appendix~\ref{app:attacks_defense}.

\paragraph{Hyperparameters.} Learning rates are tuned via grid search for each compression ratio in the non-adversarial setting ($f=0$), for both RoSDHB and \dasha. Refer to Appendix~\ref{app:dasha} for the full description of \dasha.
In the MNIST experiments, each honest worker computes the full gradient,  ensuring consistency with the theoretical analysis of RoSDHB.
For the more memory-intensive training of ResNet-18 on CIFAR-10, we use mini-batch gradient updates, as is standard in practice. RoSDHB honest workers compute mini-batch gradient updates with $b=128$ at each iteration. In its general form, \dasha allows each client to perform a full-gradient update with probability $p$ and a mini-batch update with probability $1-p$. We set the mini-batch size to $b=128$ and follow the configuration in~\cite{rammal2024communication},  by setting  $p=b/m$, where $m$ is the local dataset size. Appendix~\ref{app:experimental details} provides all hyper-parameters and computing resources.


For MNIST, we evaluate scenarios with $f \in \{1,3,5,9\}$ Byzantine attackers and sampling ratios $1/\alpha = k/d \in \{0.05, 0.1, 0.3, 0.5, 1.0\}$ for $250$ communication rounds. For CIFAR-10, we evaluate $f \in \{1,3,5\}$ and $k/d \in \{0.25, 0.5, 0.75, 1.0\}$ for $300$ communication rounds. For each combination of $f$ and $k/d$, we perform $5$ independent runs and report the mean along with its standard error. Due to space limitations, we only report $f \in \{1,3\}$ scenarios employing FOE attack strategy; the full set of experiments is provided in Appendix~\ref{app:fullset_experiments}. In the figures, \dasha is simply indicated as SOTA.

\begin{figure}[bth]
    \centering

    \begin{subfigure}[t]{\linewidth}
        \centering
        \includegraphics[width=\linewidth]{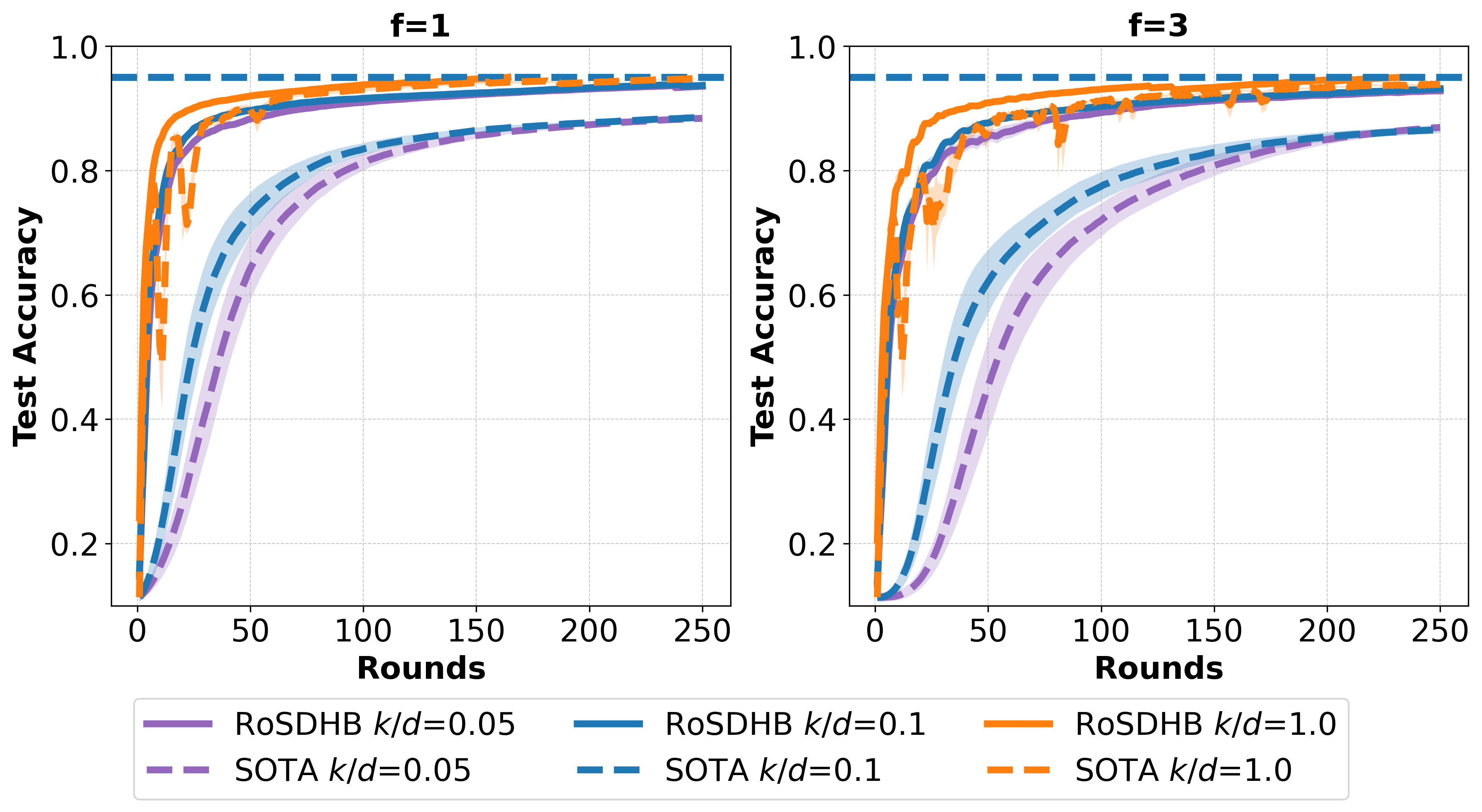}
        \caption{MNIST Dataset.}
        \label{fig:convMNIST}
    \end{subfigure}
    \hfill
    \begin{subfigure}[t]{\linewidth}
        \centering
        \includegraphics[width=\linewidth]{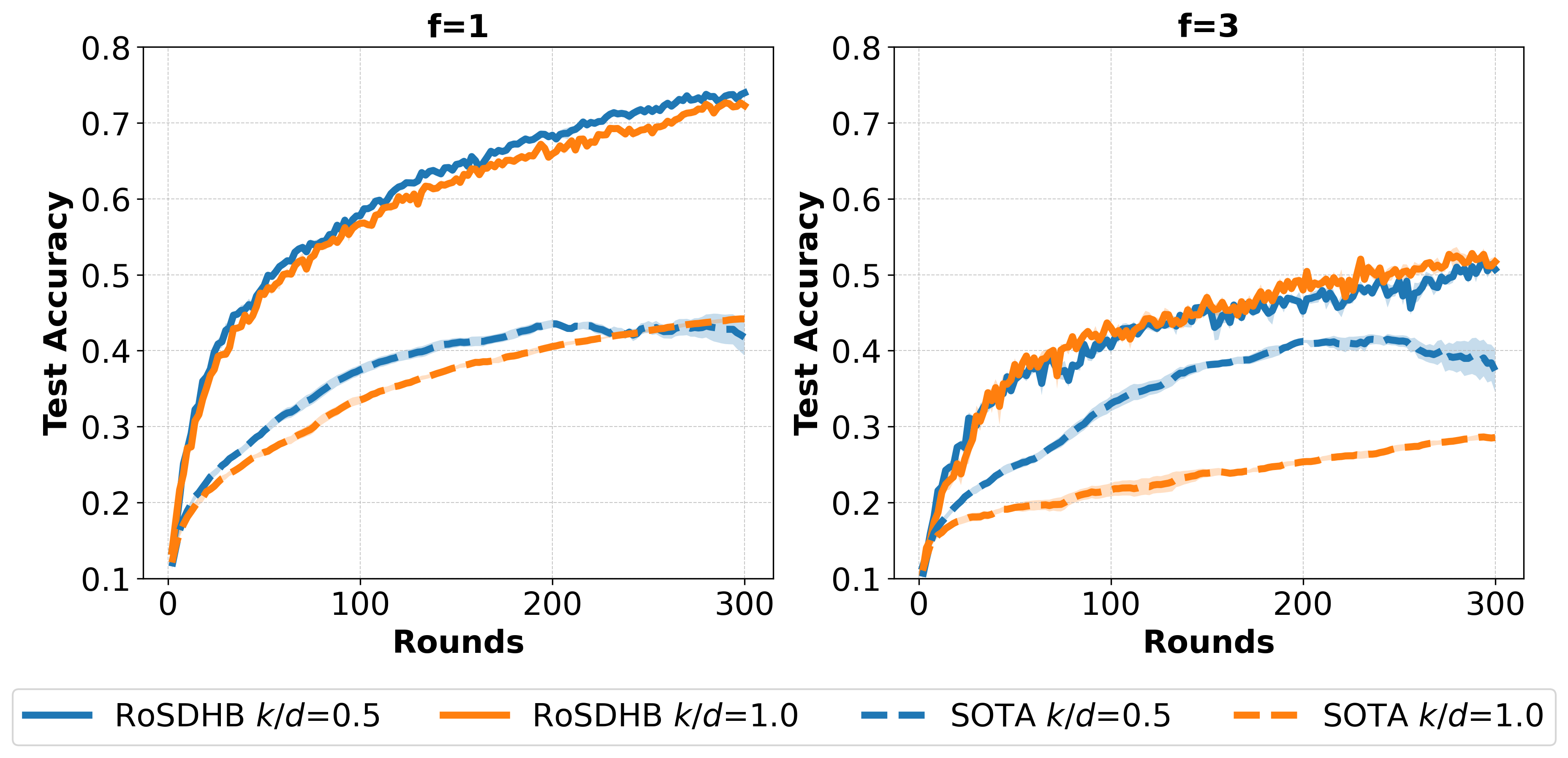}
        \caption{CIFAR-10 Dataset.}
        \label{fig:convCIFAR}
    \end{subfigure}

    \caption{The convergence plots of RoSDHB and \dasha (SOTA) under varying number of Byzantine workers $f$ and sampling ratios $k/d$.}
    \label{fig:convergence}
    \vspace{-10pt}
\end{figure}

\subsection{Empirical results}

\paragraph{Stable and fast convergence.} Figure~\ref{fig:convergence} illustrates the test accuracy convergence of our method RoSDHB (solid lines) compared with  \dasha (dashed lines) under different numbers of Byzantine attackers $f$ and sampling ratios $k/d$ for both MNIST and CIFAR-10 datasets. Across all experimental settings, RoSDHB exhibits faster and more stable convergence than \dasha. 
For instance, for $k/d \le 0.1$,
RoSDHB achieves over $92.5\%$ test accuracy on MNIST by the end of training, compared to only about $88.5\%$ for the  baseline.  On CIFAR-10, RoSDHB attains over $70\%$ test accuracy, whereas \dasha reaches only around $44\%$.
Moreover, the convergence curves of RoSDHB remain relatively close across different $k/d$ values, indicating that its performance is less sensitive to the sampling ratio than the baseline. 



\begin{figure}
    \centering
    \includegraphics[width=\linewidth]{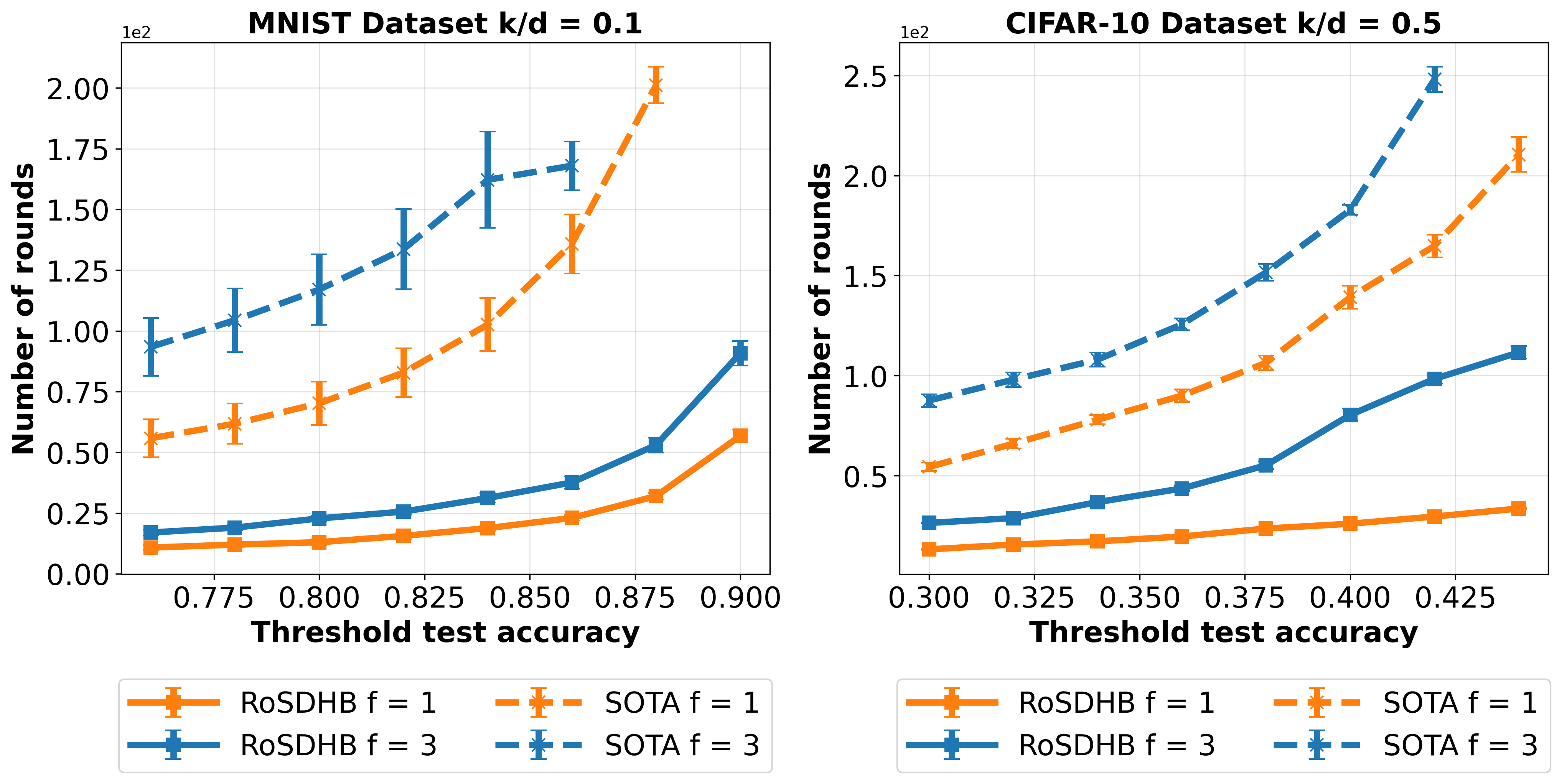}
    \caption{Comparison of \textsc{RoSDHB} and \dasha (SOTA) convergence time required to reach various test threshold accuracy under varying Byzantine workers $f\in\{1,3 \}$.} 
    \label{fig:varying_threshold}
\end{figure}

\paragraph{Speedup and robustness.} In Figure~\ref{fig:varying_threshold}, we present the number of communication rounds required to reach different test accuracies under varying numbers of Byzantine attackers $f$. As expected, the number of communication rounds grows with $f$, reflecting the theoretical dependence of the convergence rate on $\kappa$, which itself increases with the number of Byzantine attackers~$f$. 
Overall, RoSDHB achieves more than $4.5\times$ speedup on MNIST with $k/d=0.1$ and $2.3\times$ on CIFAR-10 with $k/d=0.5$, where speedup is measured as the reduction in communication rounds needed to reach a given test accuracy compared to \dasha. The exact speedup values on different threshold values are reported in Tables~\ref{tab:mnist_speedup_table} and~\ref{tab:cifar_speedup_table} in Appendix~\ref{app:further_results}.
Furthermore, RoSDHB shows stronger robustness to Byzantine attackers: with $f=3$, it still outperforms the baseline with only $f=1$. By contrast, the baseline with $f=3$ fails to reach the target thresholds of $88\%$ on MNIST and $44\%$ on CIFAR-10 within the allotted communication rounds.


\begin{figure}[h!]

    \begin{subfigure}[t]{\linewidth}
        \centering
        \includegraphics[width=\linewidth]{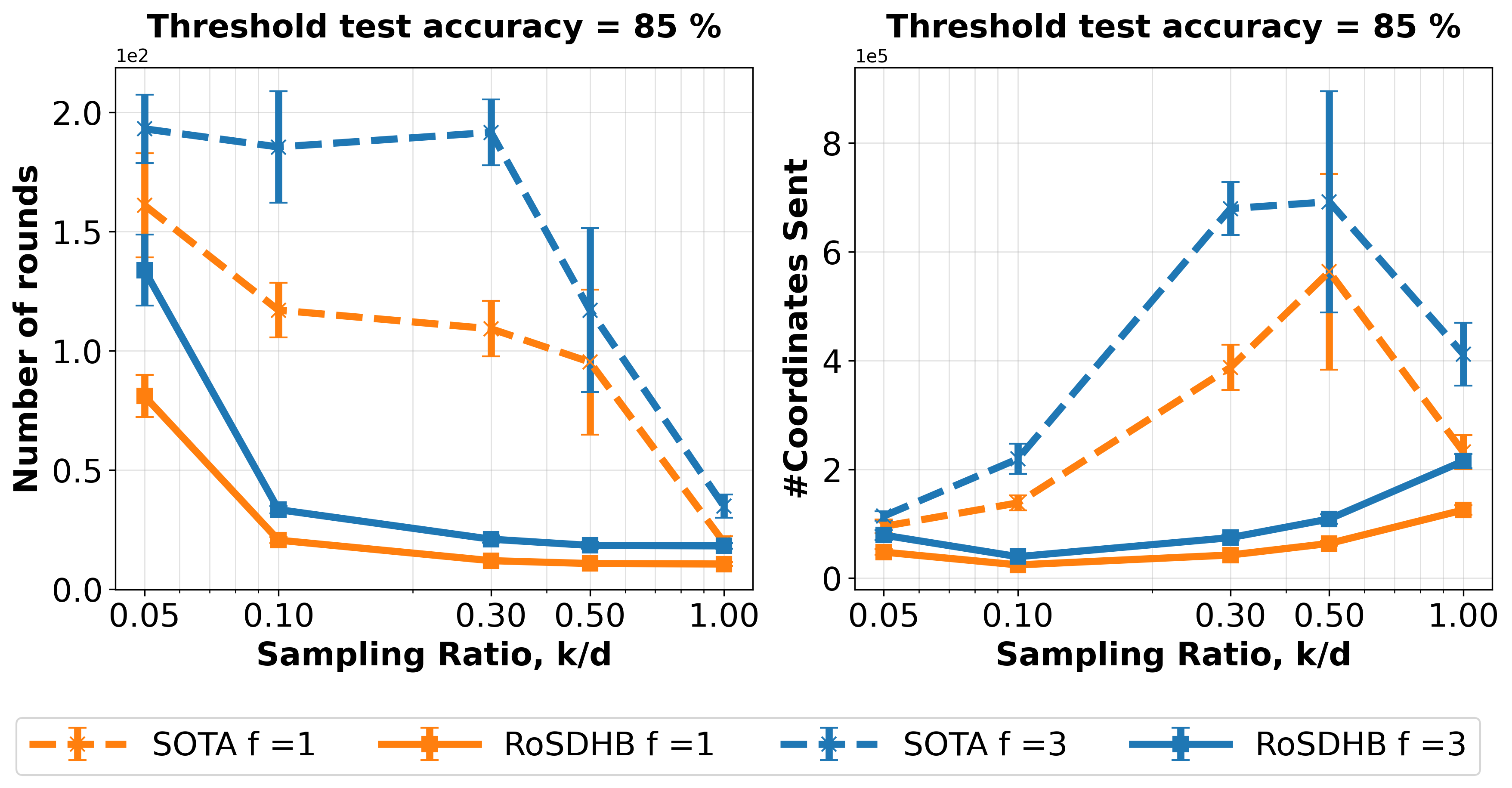}
        \caption{MNIST Dataset.}
        \label{fig:commMNIST}
    \end{subfigure}
    \hfill

    
    \begin{subfigure}[t]{\linewidth}
        \centering
        \includegraphics[width=\linewidth]{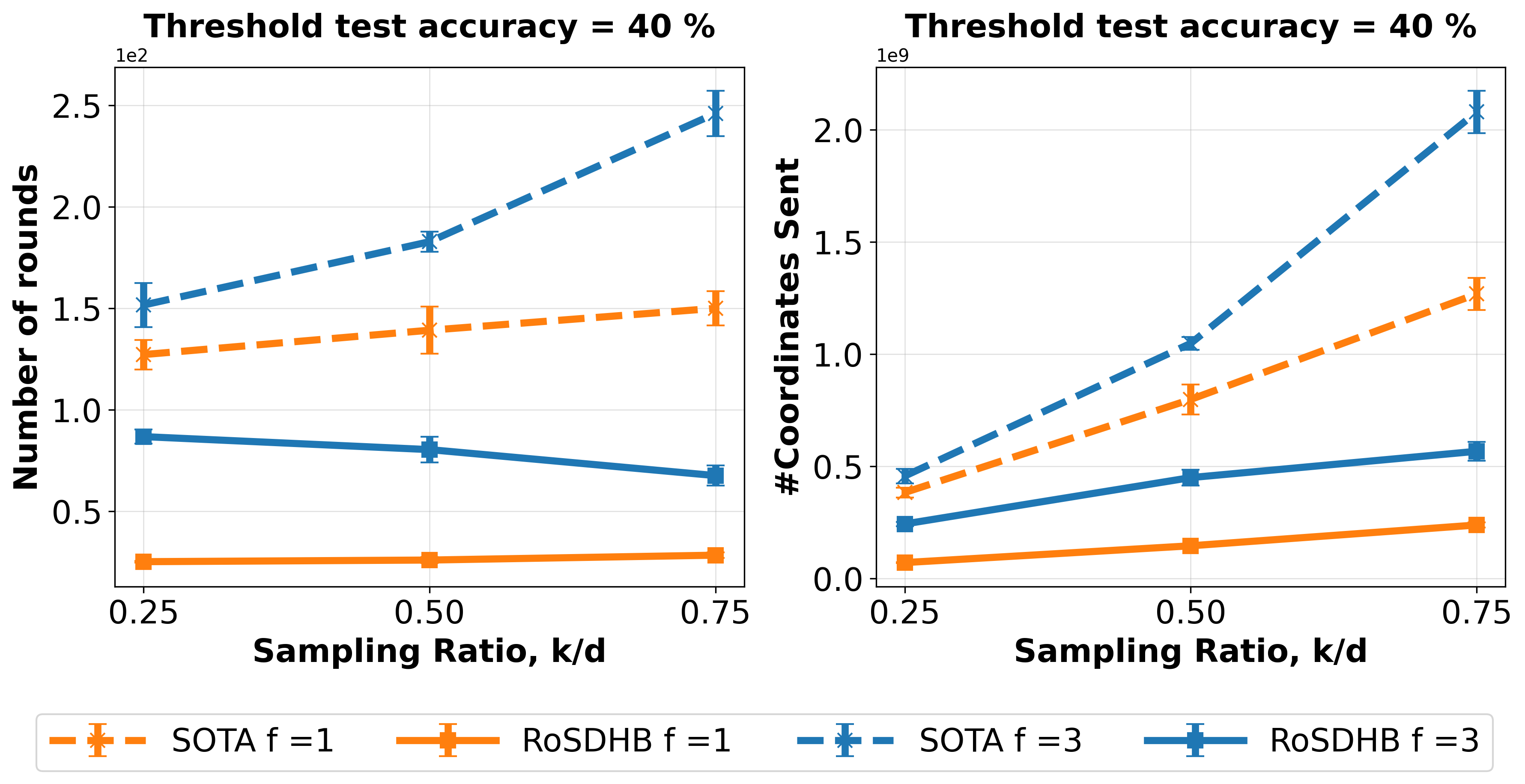}
        \caption{CIFAR-10 Dataset.}
        \label{fig:commCIFAR}
    \end{subfigure}

    \caption{Comparison of convergence time and communication cost between \textsc{RoSDHB} and \dasha (SOTA) under varying sampling ratios and number of Byzantine workers \(f \in \{1,3\}\).} 
    \label{fig:communication_mnist_cifar10}
\end{figure}

\paragraph{Compression \& communication efficiency.} Figure~\ref{fig:communication_mnist_cifar10} reports the performance of RoSDHB and the baseline, measured by the number of rounds and the number of coordinates transmitted per client to reach a target test accuracy, under different sampling ratios $k/d$. 
For RoSDHB, the number of rounds increases as the sampling ratio $k/d$ ($=1/\alpha$) decreases. The interplay between compression and robustness is also evident: at smaller sampling ratios, RoSDHB becomes more vulnerable to Byzantine attackers, leading to a comparatively larger increase in the number of rounds. Both effects are consistent with our theoretical analysis, in particular Corollary~\ref{co:rosdhb}. 
Interestingly, although smaller sampling ratios require more rounds, the reduced number of coordinates transmitted per round can completely offset this increase (Fig.~\ref{fig:commCIFAR}, right), or lead to non-monotonic behavior (Fig.~\ref{fig:commMNIST}, right).
By contrast, the results for \dasha are less aligned with its theoretical analysis in~\cite{rammal2024communication}.
While the expected trend holds for MNIST, with more rounds needed as the sampling ratio decreases, CIFAR-10 exhibits the opposite behavior. 
Furthermore, sensitivity to Byzantine attackers does not consistently increase at smaller sampling ratios.
In any case, RoSDHB achieves substantial communication savings compared to the  baseline across both datasets.
For instance, RoSDHB reduces the number of transmitted coordinates by $89\%$ relative to the baseline at $k/d=0.3$ on MNIST. More detailed communication saving statistics are reported in Tables~\ref{tab:convergence_time_coords_speedup_savings} and~\ref{tab:cifar10_comm_time_bits_savings} in Appendix~\ref{app:further_results}. 

\paragraph{Resilience to data heterogeneity.} Finally, we examine the performance of RoSDHB and \dasha under varying levels of data heterogeneity among honest workers. To this end, we control the Dirichlet distribution parameter $w$ for data partitioning, ranging from a highly heterogeneous regime ($w=0.5$) to a more moderate regime ($w=5$).
\begin{figure}[htb!]
  \begin{center}
    \includegraphics[width=0.9\linewidth]{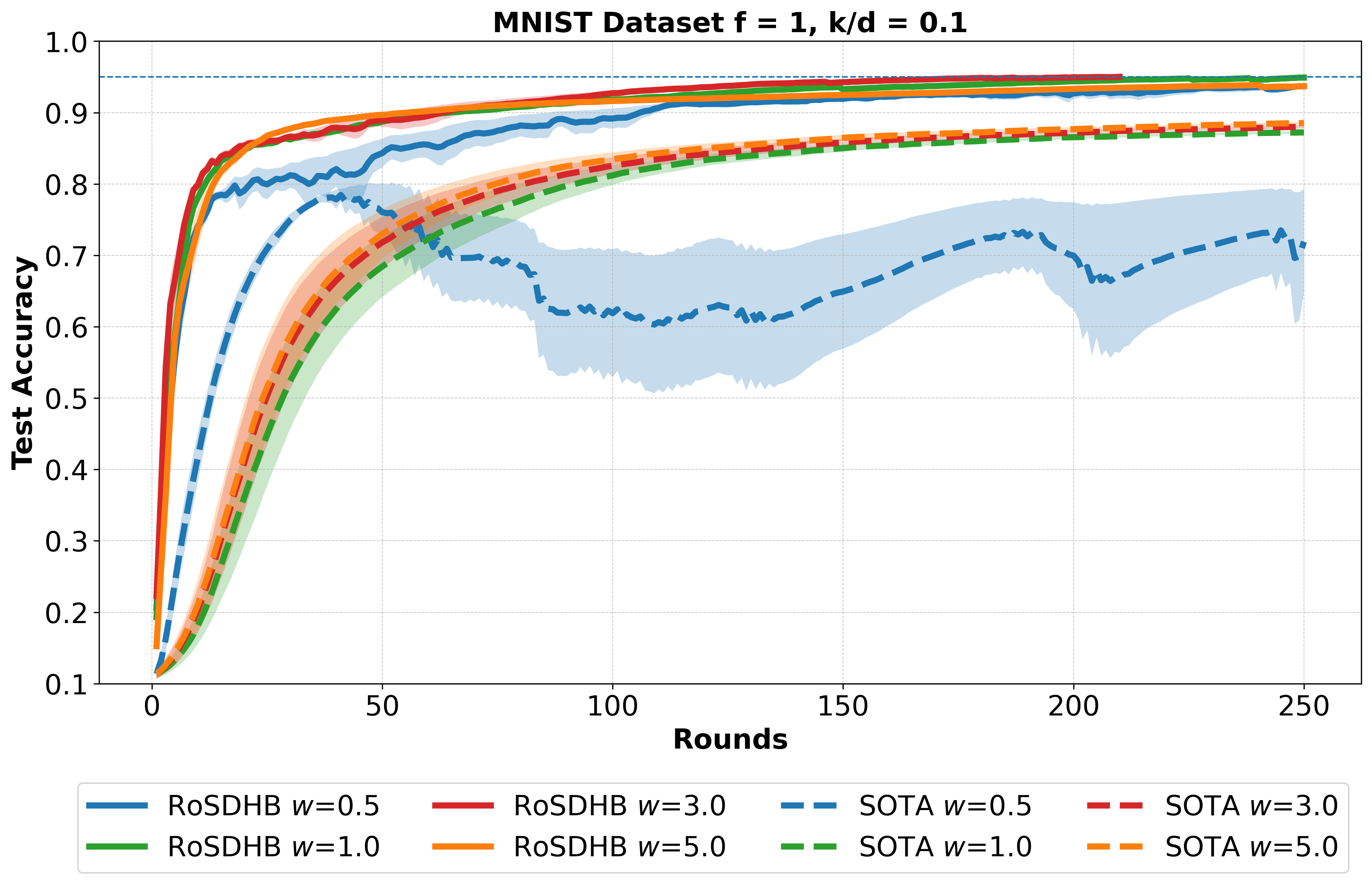}
  \end{center}
    \caption{The convergence plot of RoSDHB and \dasha (SOTA) on MNIST with $k/d=0.1$ under varying data heterogeneity.}
    \label{fig:data_hetero_effect}
\end{figure}
Figure~\ref{fig:data_hetero_effect} shows that RoSDHB consistently outperforms the  baseline across all levels of data heterogeneity. Moreover, RoSDHB demonstrates greater resilience to heterogeneity: in the highly heterogeneous regime ($w=0.5$), it achieves a test accuracy of $93.7\%$, whereas the  baseline reaches only $71.9\%$.

The relative performance of RoSDHB and \dasha is consistent under the ALIE attack (Appendix~\ref{app:fullset_experiments}). Overall, these results highlight the robustness and scalability of \textsc{RoSDHB} across a wide range of compression ratios, Byzantine worker fractions, data heterogeneity levels, and benchmark datasets. 
Our proposed method not only sustains high accuracy under compression and adversarial conditions but also achieves significant improvements in communication efficiency. 


\section{Concluding Remarks}\label{sec:conc_rw}
We have proposed RoSDHB, a new distributed learning algorithm that achieves Byzantine-robustness while enforcing communication compression. By design, RoSDHB reduces both memory and communication costs at the workers compared to the SOTA Byz-DASHA-PAGE~\cite{rammal2024communication}. Our theoretical analysis shows that the classical heavy-ball momentum, when applied after gradient sparsification and combined with a coordinated (global) sparsification scheme, is sufficient to guarantee tight Byzantine-robustness under standard assumptions. Beyond these guarantees, our experiments confirm that RoSDHB consistently outperforms Byz-DASHA-PAGE in practice, providing stronger resilience to Byzantine attackers together with substantial communication savings.

A natural future direction is to extend our framework to distributed stochastic gradient descent, moving beyond the full-gradient setting considered here. Other promising extensions include asynchronous training, adaptive sparsification masks, and more general compression schemes. Another open line of work is to explore the interplay of RoSDHB with privacy-preserving mechanisms such as differentially private Gaussian noise. Finally, motivated by federated learning, we plan to investigate the behavior of RoSDHB under multiple local SGD steps at each round.

\bibliographystyle{apalike}
\bibliography{references}

\newpage

\section*{Checklist}



\begin{enumerate}

  \item For all models and algorithms presented, check if you include:
  \begin{enumerate}
    \item A clear description of the mathematical setting, assumptions, algorithm, and/or model. [Yes]
    \item An analysis of the properties and complexity (time, space, sample size) of any algorithm. [Yes]
    \item (Optional) Anonymized source code, with specification of all dependencies, including external libraries. [Yes]
  \end{enumerate}

  \item For any theoretical claim, check if you include:
  \begin{enumerate}
    \item Statements of the full set of assumptions of all theoretical results. [Yes]
    \item Complete proofs of all theoretical results. [Yes]
    \item Clear explanations of any assumptions. [Yes]     
  \end{enumerate}

  \item For all figures and tables that present empirical results, check if you include:
  \begin{enumerate}
    \item The code, data, and instructions needed to reproduce the main experimental results (either in the supplemental material or as a URL). [Yes]
    \item All the training details (e.g., data splits, hyperparameters, how they were chosen). [Yes]
    \item A clear definition of the specific measure or statistics and error bars (e.g., with respect to the random seed after running experiments multiple times). [Yes]
    \item A description of the computing infrastructure used. (e.g., type of GPUs, internal cluster, or cloud provider). [Yes]
  \end{enumerate}

  \item If you are using existing assets (e.g., code, data, models) or curating/releasing new assets, check if you include:
  \begin{enumerate}
    \item Citations of the creator If your work uses existing assets. [Yes]
    \item The license information of the assets, if applicable. [Not Applicable]
    \item New assets either in the supplemental material or as a URL, if applicable. [Not Applicable]
    \item Information about consent from data providers/curators. [Not Applicable]
    \item Discussion of sensible content if applicable, e.g., personally identifiable information or offensive content. [Not Applicable]
  \end{enumerate}

  \item If you used crowdsourcing or conducted research with human subjects, check if you include:
  \begin{enumerate}
    \item The full text of instructions given to participants and screenshots. [Not Applicable]
    \item Descriptions of potential participant risks, with links to Institutional Review Board (IRB) approvals if applicable. [Not Applicable]
    \item The estimated hourly wage paid to participants and the total amount spent on participant compensation. [Not Applicable]
  \end{enumerate}

\end{enumerate}

\appendix
\onecolumn
\aistatstitle{Reconciling Communication Compression and Byzantine-Robustness in Distributed Learning: 
Supplementary Materials}



\section{About the Global Hessian-Variance Assumption}
\label{app:ghv}
\begin{assumption} (Global Hessian-Variance).
There exists a constant $L_{\pm}\ge 0$ such that for all $x,y\in\R^d$,
\begin{equation}
\label{eq:ghv}
\frac{1}{|\mathcal H|}\sum_{i \in \mathcal H} \bigl\|\nabla \mathcal L_i(x)-\nabla \mathcal L_i(y)\bigr\|^2
\;-\;
\bigl\|\nabla \mathcal L(x)-\nabla f(y)\bigr\|^2
\;\le\;
L_{\pm}^2 \,\|x-y\|^2.
\end{equation}
\label{assumption:ghv}
\end{assumption}

In what follows, we provide a one-dimensional, two-worker construction in which the global objective $\mathcal L_H$ is $L$-smooth, yet \eqref{eq:ghv} fails for any finite $L_{\pm}$.

\begin{lemma}[Global $L$-smoothness without global Hessian-variance]
\label{lem:counterexample}
Fix $L>0$. Define
\[
\mathcal L_1(x)=\tfrac{L}{2}x^2+\tfrac{1}{12}x^4,\qquad
\mathcal L_2(x)=\tfrac{L}{2}x^2-\tfrac{1}{12}x^4.
\]
Then the average objective $\mathcal L=\tfrac12(\mathcal L_1+\mathcal L_2)$ is globally $L$-smooth (indeed $f''(x)\equiv L$), but the hypothesis \eqref{eq:ghv} does not hold for any finite $L_{\pm}$.
\end{lemma}

\begin{proof}
Since $f''(x)=\tfrac12(f_1''(x)+f_2''(x))=\tfrac12\bigl((L+x^2)+(L-x^2)\bigr)=L$, the function $f$ is globally $L$-smooth.

Let $\Delta_i(x,y):=\nabla \mathcal L_i(x)-\nabla \mathcal L_i(y)$. For $G=2$, the left hand side in \eqref{eq:ghv} simplifies to
\[
\frac12\bigl(\Delta_1^2+\Delta_2^2\bigr)-\Bigl(\tfrac{\Delta_1+\Delta_2}{2}\Bigr)^2
=\frac{1}{4}\bigl(\Delta_1-\Delta_2\bigr)^2.
\]
Moreover,
\[
\textstyle \Delta_1-\Delta_2
=\frac{2}{3}\bigl(x^3-y^3\bigr).
\]
If \eqref{eq:ghv} held, then for the pairs $(x,-x)$ one would have
\[
\textstyle \frac{4}{9}x^6
\;\le\;
L_{\pm}^2\,\|x-(-x)\|^2
=4L_{\pm}^2 x^2,
\qquad\text{i.e.}\qquad
\frac{1}{9}x^4\le L_{\pm}^2
\quad\forall\,x\in\R,
\]
which is impossible for any finite $L_{\pm}$ as $|x|\to\infty$. Therefore, \eqref{eq:ghv} fails.
\end{proof}


\newpage
\section{Proofs for \textsc{RoSDHB}}\label{app:glorhb}
    \noindent \textbf{Notation.} Let $\condexpect{k}{\cdot}$ denote the expectation over the random sparsification. Recall from Step (1) of Algorithm~\ref{alg:rghb} that the RandK sparsifier $\mask(k) \in \{0, 1\}^d$ is generated by the server and sent along with the model at the beginning of each iteration. This is used by all the workers in Step 3(b) for compressing the locally computed gradients. Thus, the expectation is over the independent random sampling of $k$ coordinates, out of $d$ coordinates, set by the server in a given iteration.

\medskip

\noindent We first obtain some lemmas, which help bound the impact of global sparsification later in our convergence analysis.

\begin{lemma}\label{lem:helperMmtGDGlobalG0}
    Consider Algorithm \ref{alg:rghb}. For the set of honest workers $\mathcal{H}$, let $\sparHonest{t} = \frac{1}{|\mathcal{H}|}\sum_{i\in \mathcal{H}} \tilde{g}_i^t$, denote the mean of the sparsified gradients of the honest workers at time $t$. Then, for any $t\in[T]$, we get
        $$\textstyle \condexpect{k}{\norm{\sparHonest{t} - \lossHonest{t-1}}^2} \leq \left(\frac{d}{k}-1 \right) \norm{\lossHonest{t-1}}^2 .$$
\end{lemma}
\begin{proof}
     Recall  from Algorithm~\ref{alg:rghb} that all workers use the same set of sparsified bits in an iteration for compression, which were determined by the server. Fix a time $t\in [T]$, and let $\mask^t(k)$ be the mask set by the server for iteration $t$. Then, 
     $$\textstyle \sparHonest{t} = \frac{1}{|\mathcal{H}|}\sum_{i\in \mathcal{H}} \tilde{g}_i^t =  \frac{d}{k}\frac{1}{|\mathcal{H}|}\sum_{i\in \mathcal{H}} g_i^t \odot \mask^t(k) = \frac{d}{k}\left(\frac{1}{|\mathcal{H}|}\sum_{i\in \mathcal{H}} \lossworker{i}{t-1}\right) \odot \mask^t(k) =\frac{d}{k}\lossHonest{t-1} \odot \mask^t(k).$$ 
     Using the above, we have
        \begin{flalign*}
            \textstyle \norm{\sparHonest{t} - \lossHonest{t-1}}^2 &= \textstyle \norm{\frac{d}{k}\lossHonest{t-1}\odot \mask^t(k) - \lossHonest{t-1}}^2 .
        \end{flalign*} 
    Taking expectation over the random sparsification, $\condexpect{k}{\cdot}$ on both sides above, we get
        \begin{flalign*}
            \condexpect{k}{\norm{\sparHonest{t} - \lossHonest{t-1}}^2} &\leq \textstyle \condexpect{k}{\norm{\frac{d}{k}\lossHonest{t-1} \odot \mask^t(k) - \lossHonest{t-1}}^2}.
        \end{flalign*}
    Next, using the unbiased-ness of the random sparsifier, we get
    \begin{flalign*}
     \condexpect{k}{\norm{\sparHonest{t} - \lossHonest{t-1}}^2}  
        &\leq \textstyle \frac{k}{d}{\norm{\frac{d}{k}\lossHonest{t-1} - \lossHonest{t-1}}^2} + \left(1-\frac{k}{d}\right){\norm{ \lossHonest{t-1}}^2} \\
        &\leq \textstyle \frac{k}{d}\left( \frac{d}{k}-1\right)^2\norm{\lossHonest{t-1}}^2 + \left(1-\frac{k}{d}\right)\norm{\lossHonest{t-1}}^2\\
        &= \textstyle\left(\left(1-\frac{k}{d}\right)\left(\frac{d}{k}-1\right) + \left(1-\frac{k}{d}\right)\right)\norm{\lossHonest{t-1}}^2\\
        &= \textstyle \left(\frac{d}{k}-1\right)\norm{\lossHonest{t-1}}^2,
    \end{flalign*}
    which is the desired inequality.
\end{proof}

\noindent Next, we obtain some useful lemmas similar to previous work \cite{farhadkhani2022byzantine}. Let $\mmean = \frac{1}{|\honest|} \sum_{i\in \honest} m_i^t$.

\begin{lemma}\label{lem:mmtDriftGlobalGDG0}
    Consider Algorithm \ref{alg:rghb}. Let $\mmmtdrifti{i}{t} \coloneqq m_i^t - \mmean$ and $\Upsilon^t_\mathcal{H} \coloneqq \frac{1}{|\mathcal{H}|}\sum_{i\in \mathcal{H}}\norm{\mmmtdrifti{i}{t}}^2$. Then, for any $t\in[T]$, we get
    \begin{align*}
    \textstyle \Upsilon^{t}_\mathcal{H}  &\leq \textstyle \beta\Upsilon^{t-1}_\mathcal{H} + \left(\left(1-\beta\right)^2\frac{d}{k} + \beta(1-\beta)\right)\left(G^2 + B^2\norm{ \lossHonest{t-1}}^2\right) .
    \end{align*}
\end{lemma}
\begin{proof}
    Consider an arbitrary $t \in [T]$ and an arbitrary $i \in \mathcal{H}$. Recall that  $\mmean = \frac{1}{|\honest|} \sum_{i\in \honest} m_i^t$ and  $m_i^t = \beta m_i^{t-1} + (1-\beta)\tilde{g}_i^t$, thus
    \begin{flalign}\label{eq:mmtmeanCase2B}
        \mmean &= \textstyle \frac{1}{|\mathcal{H}|}\sum_{i\in\mathcal{H}} m_i^t = \frac{1}{|\mathcal{H}|}\sum_{i\in\mathcal{H}}\left(\beta m_i^{t-1} + (1-\beta)\tilde{g}_i^t\right) = \beta  \overline{m}_{\mathcal{H}}^{t-1} + (1-\beta)\overline{\tilde{g}}_\mathcal{H}^t  
    \end{flalign}
    Let $\mmmtdrifti{i}{t} = m_i^t - \overline{m}_\mathcal{H}^t$. Then, substituting value of $\mmean$ from above and $m_i^t = \beta m_i^{t-1} + (1-\beta)\tilde{g}_i^t$,
    \begin{flalign}\label{eq:mmtDevGlobalG0}
        \mmmtdrifti{i}{t} &= \left(\beta m_i^{t-1} + (1-\beta)\tilde{g}_i^t\right) - \left(\beta \overline{m}_\mathcal{H}^{t-1} - (1-\beta) \overline{\tilde{g}}_\mathcal{H}^t \right) = \beta \mmmtdrifti{i}{t-1} +(1-\beta)(\tilde{g}_i^t - \overline{\tilde{g}}_\mathcal{H}^t)
    \end{flalign}
    \noindent Therefore,
    \begin{flalign*}
        \norm{\mmmtdrifti{i}{t}}^2 &= \norm{\beta \mmmtdrifti{i}{t-1} +(1-\beta)(\tilde{g}_i^t - \overline{\tilde{g}}_\mathcal{H}^t)}^2  \\
        &\leq \norm{\beta \mmmtdrifti{i}{t-1}}^2 + \norm{(1-\beta)(\tilde{g}_i^t - \overline{\tilde{g}}_\mathcal{H}^t)}^2 + 2\beta(1-\beta)\iprod{\mmmtdrifti{i}{t-1}}{(\tilde{g}_i^t - \overline{\tilde{g}}_\mathcal{H}^t)} \\
        &\leq \beta^2\norm{\mmmtdrifti{i}{t-1}}^2 + \left(1-\beta\right)^2\norm{(\tilde{g}_i^t - \overline{\tilde{g}}_\mathcal{H}^t)}^2 + 2\beta(1-\beta)\iprod{\mmmtdrifti{i}{t-1}}{(\tilde{g}_i^t - \overline{\tilde{g}}_\mathcal{H}^t)} 
    \end{flalign*}
    Taking conditional expectation over the k-sparsifier, then using linearity of expectation and deterministic nature of $m_i^t$, we get:
    \begin{flalign} \label{eq:localvar}
    &\condexpect{k}{\norm{\mmmtdrifti{i}{t}}^2} \leq \beta^2\norm{\mmmtdrifti{i}{t-1}}^2 + \left(1-\beta\right)^2\condexpect{k}{\norm{\tilde{g}_i^t - \overline{\tilde{g}}_\mathcal{H}^t}^2} + 2\beta(1-\beta)\iprod{\mmmtdrifti{i}{t-1}}{\condexpect{k}{\tilde{g}_i^t - \overline{\tilde{g}}_\mathcal{H}^t}}
    \end{flalign}
    Note that global sparsification uses the same random mask for all workers, thus due to unbiasedness of the random sparsifier, we have 
    \begin{align*}
        \textstyle \condexpect{k}{\tilde{g}_i^t - \overline{\tilde{g}}_\mathcal{H}^t} = g_i^t - \overline{g}^t_\mathcal{H} \quad \text{and} \quad \condexpect{k}{\norm{\tilde{g}_i^t - \overline{\tilde{g}}_\mathcal{H}^t}^2} \leq \frac{d}{k}\norm{g_i^t - \overline{g}^t_\mathcal{H}}^2 .
    \end{align*}
    Recall that $g_i^t = \lossworker{i}{t-1}$ and $\overline{g}_\mathcal{H}^t = \lossHonest{t-1}$. Using this above, we get
    \begin{align}\label{eq:spargardGlobal}
        \textstyle \condexpect{k}{\tilde{g}_i^t - \overline{\tilde{g}}_\mathcal{H}^t} &= \lossworker{i}{t-1} - \lossHonest{t-1} , \nonumber\\
        \condexpect{k}{\norm{\tilde{g}_i^t - \overline{\tilde{g}}_\mathcal{H}^t}^2} &\leq \textstyle \frac{d}{k}\norm{\lossworker{i}{t-1} - \lossHonest{t-1}}^2 .
    \end{align}
    Substituting from above in (\ref{eq:localvar}),
    \begin{flalign*}
    \norm{\mmmtdrifti{i}{t}}^2 
        &\leq \textstyle \beta^2\norm{\mmmtdrifti{i}{t-1}}^2 + \left(1-\beta\right)^2\frac{d}{k}\norm{\lossworker{i}{t-1} - \lossHonest{t-1}}^2\\
        &+ 2\beta(1-\beta)\iprod{\mmmtdrifti{i}{t-1}}{\lossworker{i}{t-1} - \lossHonest{t-1}} .
    \end{flalign*}
    Due to Cauchy-Schwartz inequality, $2\iprod{a}{b} \leq 2\norm{a}\norm{b} \leq \norm{a}^2 + \norm{b}^2$. Thus, from above we obtain that
    \begin{flalign*}
    \norm{\mmmtdrifti{i}{t}}^2 &\leq \textstyle \beta^2\norm{\mmmtdrifti{i}{t-1}}^2 + \left(1-\beta\right)^2\frac{d}{k}\norm{\lossworker{i}{t-1} - \lossHonest{t-1}}^2\\
        &+ \beta(1-\beta)\left(\norm{\mmmtdrifti{i}{t-1}}^2 + \norm{\lossworker{i}{t-1} - \lossHonest{t-1}}^2\right) .
    \end{flalign*}
    Upon re-arranging the right-hand side, we obtain that
    \begin{flalign*}
        \norm{\mmmtdrifti{i}{t}}^2 &\leq \textstyle \beta\norm{\mmmtdrifti{i}{t-1}}^2 + \left(\left(1-\beta\right)^2\frac{d}{k} + \beta(1-\beta)\right)\norm{\lossworker{i}{t-1} - \lossHonest{t-1}}^2 .
    \end{flalign*}
    Since the above holds for all $i \in \mathcal{H}$, taking the average on both sides over all $i \in \mathcal{H}$ yields 
    \begin{align*}
        \textstyle \frac{1}{\mathcal{H}}\sum_{i\in \mathcal{H}}\norm{\mmmtdrifti{i}{t}}^2 
        &\leq \textstyle \beta\left(\frac{1}{\mathcal{H}}\sum_{i\in \mathcal{H}}\norm{\mmmtdrifti{i}{t-1}}^2\right) + \Bigl(\left(1-\beta\right)^2\frac{d}{k} + \beta(1-\beta)\Bigr)\left(\frac{1}{\mathcal{H}}\sum_{i\in \mathcal{H}}\norm{\lossworker{i}{t-1} - \lossHonest{t-1}}^2\right) .
    \end{align*}
    Substituting $\Upsilon^t_\mathcal{H} = \frac{1}{\mathcal{H}}\sum_{i\in \mathcal{H}}\norm{\mmmtdrifti{i}{t}}^2 $ above, we obtain that
    \begin{align*}
    \textstyle \Upsilon^t_\mathcal{H} 
        &\leq \textstyle \beta\Upsilon^{t-1}_\mathcal{H} + \left(\left(1-\beta\right)^2\frac{d}{k} + \beta(1-\beta)\right)\left(\frac{1}{\mathcal{H}}\sum_{i\in \mathcal{H}}\norm{\lossworker{i}{t-1} - \lossHonest{t-1}}^2\right) .
    \end{align*}
    Finally, under the $(G, B)$-gradient dissimilarity property from Definition \ref{assumption:gradientGB}, we have
    \begin{align*}
    \textstyle \Upsilon^t_\mathcal{H} 
        &\leq \textstyle \beta\Upsilon^{t-1}_\mathcal{H} + \left(\left(1-\beta\right)^2\frac{d}{k} + \beta(1-\beta)\right)\left(G^2 + B^2\norm{ \lossHonest{t-1}}^2\right) ,
    \end{align*}
    which proves the lemma.
\end{proof}

\begin{lemma}\label{lem:MmtDriftGlobalG0}
     Consider Algorithm \ref{alg:rghb}, with $F$ being $(f,\kappa)$-robust. Recall that $R^t \coloneqq F(m_1^t,...,m_n^t)$. Denote $\mathcal{M}_i^t \coloneqq m_i^t - \mmean$ 
     $\Upsilon^t_\mathcal{H} \coloneqq \frac{1}{|\mathcal{H}|}\sum_{i\in \mathcal{H}}\norm{\mmmtdrifti{i}{t}}^2$, and $\xi^t \coloneqq R^t - \mmean$. Then, for each $t\in[T]$, we have:
     $$\textstyle \norm{\xi^t}^2 \leq \kappa \left( \beta\Upsilon^{t-1}_\mathcal{H} + \left(\left(1-\beta\right)^2\frac{d}{k} + \beta(1-\beta)\right)\left(G^2 + B^2\norm{ \lossHonest{t-1}}^2\right)\right) .$$
\end{lemma}
\begin{proof} 
    Fix an arbitrary $t \in [T]$. Since $F$ is an $(f,\kappa)$-robust aggregator, then by Definition \ref{def:robust_aggregation}, we have:
    \begin{equation*}
        \textstyle \norm{R^t - \mmean}^2 \leq \frac{\kappa}{|\mathcal{H}|}\sum_{i \in \mathcal{H}}\norm{m_i^t - \mmean}^2 = \frac{\kappa}{|\mathcal{H}|}\sum_{i \in \mathcal{H}}\norm{\mmmtdrifti{i}{t}}^2  = \kappa \Upsilon^t_\mathcal{H} .
    \end{equation*}
    Using Lemma~\ref{lem:mmtDriftGlobalGDG0} above, we get the desired inequality.
\end{proof}

\noindent Next, we bound the momentum deviation of the honest workers from the mean gradient of loss of the honest workers for the case of global sparsification under the $(G,B)$-gradient dissimilarity condition.

\begin{lemma}\label{lem:mmtDevGlobalG0}
    Suppose Assumption \ref{assumption:smoothness} holds true. Denote $\delta^t \coloneqq \mmean - \lossHonest{t-1}$. Consider Algorithm \ref{alg:rghb} with $T>1$, then for all $t\geq 2$, we have:
    \begin{align*}
        \expect{\norm{\delta^t}^2} &\leq \textstyle \zeta \expect{\norm{\delta^{t-1}}^2} + 2\gamma L\left( 2\gamma L (1-\beta)^2\left(\frac{d}{k}-1\right) + (1+\gamma L)\beta^2\right)\expect{\norm{\xi^{t-1}}^2}\\
            &\hspace{10pt}\textstyle + \left( 2(4\gamma^2L^2+1)(1-\beta)^2\left(\frac{d}{k}-1\right) + 4\gamma L (1+\gamma L)\beta^2\right) \expect{\norm{\lossHonest{t-2}}^2},
    \end{align*}

    where $\zeta = \left(8\gamma^2 L^2 (1-\beta)^2\left(\frac{d}{k}-1\right) + (1+4\gamma L)(1+\gamma L)\beta^2 \right)$. 
\end{lemma}
\begin{proof}
    Fix an arbitrary $t \in [T]$. Then, substituting $\mmean = \beta \overline{m}_\mathcal{H}^{t-1} + (1-\beta)\overline{\tilde{g}}_\mathcal{H}^t$ in definition of $\delta^t$:
    \begin{equation*}
        \delta^t = \mmean - \lossHonest{t-1} = \beta \overline{m}_\mathcal{H}^{t-1} + (1-\beta)\overline{\tilde{g}}_\mathcal{H}^t  - \lossHonest{t-1}
    \end{equation*}
    Upon adding and subtracting $\beta\lossHonest{t-1}$ and $\beta\lossHonest{t-2}$ on the R.H.S. above, we get:
    \begin{flalign*}
        \delta^t &= \beta \overline{m}_\mathcal{H}^{t-1} - \beta\lossHonest{t-2} + (1-\beta)\overline{\tilde{g}}_\mathcal{H}^t - \lossHonest{t-1} + \beta\lossHonest{t-1} + \beta\lossHonest{t-2}  -\beta\lossHonest{t-1} \\
        &= \beta\left(\overline{m}_\mathcal{H}^{t-1} - \lossHonest{t-2}\right) +  (1-\beta)\left(\overline{\tilde{g}}_\mathcal{H}^t - \lossHonest{t-1}\right) + \beta\left(\lossHonest{t-2} - \lossHonest{t-1}\right) 
    \end{flalign*}
    Substituting $ \delta^{t-1} = \overline{m}_\mathcal{H}^{t-1} - \beta\lossHonest{t-2}$ above:
    \begin{flalign*}
    \delta^t &= \beta\delta^{t-1} +  (1-\beta)\left(\overline{\tilde{g}}_\mathcal{H}^t - \lossHonest{t-1}\right)+ \beta\left(\lossHonest{t-2} - \lossHonest{t-1}\right) 
    \end{flalign*}
    Taking 2-norm and squaring on both side, we get
    \begin{flalign*}
        \norm{\delta^t}^2 &= \norm{\beta\delta^{t-1} +  (1-\beta)\left(\overline{\tilde{g}}_\mathcal{H}^t - \lossHonest{t-1}\right)+ \beta\left(\lossHonest{t-2} - \lossHonest{t-1}\right) }^2 \\
        &= \beta^2\norm{\delta^{t-1}}^2 + (1-\beta)^2\norm{\overline{\tilde{g}}_\mathcal{H}^t - \lossHonest{t-1}}^2 + \beta^2\norm{\lossHonest{t-2} - \lossHonest{t-1}}^2 \\
            &\text{\hspace{10pt}} + 2\beta(1-\beta)\iprod{\delta^{t-1}}{\overline{\tilde{g}}_\mathcal{H}^t - \lossHonest{t-1}} + 2\beta^2\iprod{\lossHonest{t-2} - \lossHonest{t-1}}{\delta^{t-1}}\\
            &\text{\hspace{10pt}} + 2\beta(1-\beta)\iprod{\overline{\tilde{g}}_\mathcal{H}^t - \lossHonest{t-1}}{\lossHonest{t-2} - \lossHonest{t-1}} 
    \end{flalign*}
    Taking conditional expectation over the randomness due to sparsification  $\condexpect{k}{\cdot}$ on both sides:
    \begin{flalign*}
        \condexpect{k}{\norm{\delta^t}^2} & = \textstyle \beta^2\norm{\delta^{t-1}}^2 + (1-\beta)^2\condexpect{k}{\norm{\overline{\tilde{g}}_\mathcal{H}^t - \lossHonest{t-1}}^2} + \beta^2\norm{\lossHonest{t-2} - \lossHonest{t-1}}^2 \\
            & + 2\beta(1-\beta)\iprod{\delta^{t-1}}{\condexpect{k}{\overline{\tilde{g}}_\mathcal{H}^t} - \lossHonest{t-1}} + 2\beta^2\iprod{\lossHonest{t-2} - \lossHonest{t-1}}{\delta^{t-1}}\\
            & + 2\beta(1-\beta)\iprod{\condexpect{k}{\overline{\tilde{g}}_\mathcal{H}^t} - \lossHonest{t-1}}{\lossHonest{t-2} - \lossHonest{t-1}} 
    \end{flalign*}
    Note that $\condexpect{k}{\sparHonest{t}} = \condexpect{k}{\frac{1}{|\mathcal{H}|}\sum_{i\in \mathcal{H}}\tilde{g}_i^t} = \frac{1}{|\mathcal{H}|}\sum_{i\in \mathcal{H}}\condexpect{k}{\tilde{g}_i^t} = \frac{1}{|\mathcal{H}|}\sum_{i\in \mathcal{H}} g_i^t = \frac{1}{|\mathcal{H}|}\sum_{i\in \mathcal{H}}\lossWorker{i}{t-1} = \lossHonest{t-1}$. Substituting above, we get
    \begin{flalign*}
    \condexpect{k}{\norm{\delta^t}^2} &\leq \beta^2\norm{\delta^{t-1}}^2 + (1-\beta)^2\condexpect{k}{\norm{\overline{\tilde{g}}_\mathcal{H}^t - \lossHonest{t-1}}^2}+ \beta^2\norm{\lossHonest{t-2} - \lossHonest{t-1}}^2\\
            &\text{\hspace{10pt}} + 2\beta(1-\beta)\iprod{\delta^{t-1}}{\lossHonest{t-1} - \lossHonest{t-1}} + 2\beta^2\iprod{\lossHonest{t-2} - \lossHonest{t-1}}{\delta^{t-1}}\\
            &\text{\hspace{10pt}} \textstyle + 2\beta(1-\beta)\iprod{{\lossHonest{t-1} - \lossHonest{t-1}}}{\lossHonest{t-2} - \lossHonest{t-1}}\\
        &\leq \beta^2\norm{\delta^{t-1}}^2 + (1-\beta)^2\condexpect{k}{\norm{\overline{\tilde{g}}_\mathcal{H}^t - \lossHonest{t-1}}^2} + \beta^2\norm{\lossHonest{t-2} - \lossHonest{t-1}}^2\\
         &\text{\hspace{10pt}}+ 2\beta^2\iprod{\lossHonest{t-2} - \lossHonest{t-1}}{\delta^{t-1}}
    \end{flalign*}
    \noindent Under global sparsification, using Lemma~\ref{lem:helperMmtGDGlobalG0} above we obtain that 
    \begin{align*}
        &\textstyle \condexpect{k}{\norm{\delta^t}^2} \leq \textstyle (1+\gamma L)\beta^2\norm{\delta^{t-1}}^2 +(1-\beta)^2\left(\frac{d}{k}-1\right)\norm{\lossHonest{t-1}}^2  + \gamma L (1+\gamma L)\beta^2 \norm{R^{t-1}}^2.
    \end{align*}
    Adding and subtracting $\lossHonest{t-2}$ from $\lossHonest{t-1}$ term above, we get
    \begin{align*}
        &\textstyle \condexpect{k}{\norm{\delta^t}^2} \leq \textstyle (1+\gamma L)\beta^2\norm{\delta^{t-1}}^2 +(1-\beta)^2\left(\frac{d}{k}-1\right)\norm{\lossHonest{t-1} - \lossHonest{t-2} + \lossHonest{t-2}}^2\\
        &\hspace{10pt} + \gamma L (1+\gamma L)\beta^2 \norm{R^{t-1}}^2.
    \end{align*}
    Using $\norm{a + b}^2 \leq 2\norm{a}^2 + 2\norm{b}^2$, 
    \begin{align*}
        \textstyle \condexpect{k}{\norm{\delta^t}^2} &\leq \textstyle (1+\gamma L)\beta^2\norm{\delta^{t-1}}^2 +(1-\beta)^2\left(\frac{d}{k}-1\right)\Bigl( 2\norm{\lossHonest{t-1} - \lossHonest{t-2}}^2 + 2\norm{\lossHonest{t-2}}^2\Bigr)
        \\
        &\hspace{10pt} \textstyle  + \gamma L (1+\gamma L)\beta^2 \norm{R^{t-1}}^2.
    \end{align*}
    Next, using the L-smoothness property (from Assumption \ref{assumption:smoothness}), we have $\norm{\lossHonest{t-2} - \lossHonest{t-1}} \leq L\norm{\theta_{t-2}-\theta_{t-1}}$. Substituting above, we get
    \begin{align*}
        \textstyle \condexpect{k}{\norm{\delta^t}^2} &\leq \textstyle (1+\gamma L)\beta^2\norm{\delta^{t-1}}^2 +(1-\beta)^2\left(\frac{d}{k}-1\right)\left(2L^2\norm{\theta_{t-2} - \theta_{t-1}}^2 + 2\norm{\lossHonest{t-2}}^2\right) 
        \\
        &\hspace{10pt} \textstyle  + \gamma L (1+\gamma L)\beta^2 \norm{R^{t-1}}^2 .
    \end{align*}
    Recall from Step 4 of our algorithm that 
    $\theta_{t-2} - \theta_{t-1} = \gamma R^{t-1}$. Using this and substituting $\beta_k^2 \coloneqq 2(1-\beta)^2\left(\frac{d}{k}-1\right)$ above, we obtain that
    \begin{align*}
        \textstyle \condexpect{k}{\norm{\delta^t}^2} &\leq \textstyle (1+\gamma L)\beta^2\norm{\delta^{t-1}}^2 +\beta_k^2\gamma^2 L^2\norm{R^{t-1}}^2 + \beta_k^2\norm{\lossHonest{t-2}}^2 + \gamma L (1+\gamma L)\beta^2 \norm{R^{t-1}}^2\\
        &\leq \textstyle (1+\gamma L)\beta^2\norm{\delta^{t-1}}^2  + \beta_k^2\norm{\lossHonest{t-2}}^2  + \gamma L\left( \gamma L \beta_k^2+ (1+\gamma L)\beta^2\right) \norm{R^{t-1}}^2 .
    \end{align*}
    
    \noindent Recall from Lemma \ref{lem:MmtDriftGlobalG0} that we denote $\xi^t \coloneqq R^t - \mmean$ implying $R^t \coloneqq \xi^t + \mmean$, and $\delta^t \coloneqq \overline{m}^t_\mathcal{H} - \lossHonest{t-1}$ implying $\overline{m}^t_\mathcal{H} \coloneqq \delta^t + \lossHonest{t-1} $. Thus, using Cauchy-Schwartz inequality, we have $\norm{R^{t-1}}^2 \leq 2\norm{\xi^{t-1}}^2 + 2\norm{\mmeanT{t-1}}^2 = 2\norm{\xi^{t-1}}^2 + 4\norm{\delta^{t-1}}^2 + 4\norm{\nabla\mathcal{L}_\mathcal{H}(\theta_{t-2})}^2$. Using this above, we obtain that 
    \begin{flalign*}
         \condexpect{k}{\norm{\delta^t}^2} &\leq \textstyle (1+\gamma L)\beta^2\norm{\delta^{t-1}}^2 + \beta_k^2\norm{\lossHonest{t-2}}^2 \\
            &\hspace{10pt} \textstyle + \gamma L\left( \gamma L \beta_k^2+ (1+\gamma L)\beta^2\right)\left(2\norm{\xi^{t-1}}^2 + 4\norm{\delta^{t-1}}^2 + 4\norm{\nabla\mathcal{L}_\mathcal{H}(\theta_{t-2})}^2 \right)&& \\
        & = \textstyle \left(4\gamma^2 L^2 \beta_k^2 + (1+4\gamma L)(1+\gamma L)\beta^2 \right)\norm{\delta^{t-1}}^2 + 2\gamma L\left( \gamma L \beta_k^2+ (1+\gamma L)\beta^2\right)\norm{\xi^{t-1}}^2 \\
            &\hspace{10pt} \textstyle  + \left( (4\gamma^2L^2+1)\beta^2_k + 4\gamma L (1+\gamma L)\beta^2\right) \norm{\lossHonest{t-2}}^2 .
    \end{flalign*}
    \noindent Recall that $\beta_k^2 \coloneqq 2(1-\beta)^2\left(\frac{d}{k}-1\right)$, then substituting $\zeta = \left(4\gamma^2 L^2 \beta_k^2 + (1+4\gamma L)(1+\gamma L)\beta^2 \right)= \left(8\gamma^2 L^2 (1-\beta)^2\left(\frac{d}{k}-1\right) + (1+4\gamma L)(1+\gamma L)\beta^2 \right)$ and $\beta_k$ above, we obtain that
    \begin{align*}
        \condexpect{k}{\norm{\delta^t}^2} &\leq \textstyle \zeta \norm{\delta^{t-1}}^2 + 2\gamma L\left( 2\gamma L (1-\beta)^2\left(\frac{d}{k}-1\right) + (1+\gamma L)\beta^2\right)\norm{\xi^{t-1}}^2\\
            &\hspace{10pt}\textstyle + \left( 2(4\gamma^2L^2+1)(1-\beta)^2\left(\frac{d}{k}-1\right) + 4\gamma L (1+\gamma L)\beta^2\right) \norm{\lossHonest{t-2}}^2 .
    \end{align*}
    Since $t$ above is an arbitrary value in $[T]$ greater than $1$, upon taking expectation over all time steps, on both the sides proves the lemma.
\end{proof}

We use Lemma 9 from \cite{allouah2023fixing} in our analysis for proving Theorem~\ref{thm:rosdhb}. We restate it below for completeness.

\begin{lemma}\label{lem:modelDriftDevmmtGD}
    Suppose Assumption \ref{assumption:smoothness} holds true. For all $t\in [T]$, we obtain that:
    \begin{align*}
        \expect{2\mathcal{L}_\mathcal{H}(\theta_{t})-2\mathcal{L}_\mathcal{H}(\theta_{t-1})} &\leq-\gamma(1-4\gamma L)\expect{\norm{\lossHonest{t-1}}^2} + 2\gamma(1+2\gamma L)\expect{\norm{\delta^t}^2} \\
        &\text{\hspace{10pt}}+  2\gamma(1+\gamma L)\expect{\norm{\xi^t}^2} .
    \end{align*}
\end{lemma}
\begin{proof}
    Consider an arbitrary step $t$. Recall from Assumption \ref{assumption:smoothness} that $\mathcal{L}_\mathcal{H}$ is an L-Lipshitz smooth function, hence we have:
    \begin{equation*}
        \textstyle \mathcal{L}_\mathcal{H}(\theta_t) - \mathcal{L}_\mathcal{H}(\theta_{t-1}) \leq \iprod{\theta_t-\theta_{t-1}}{\lossHonest{t-1}} + \frac{L}{2}\norm{\theta_t - \theta_{t-1}}^2.
    \end{equation*}
    Next, recall from our algorithm \ref{alg:rghb} that $\theta_t = \theta_{t-1} - \gamma R^t$, and recall from Lemma \ref{lem:MmtDriftGlobalG0} that we denote $\xi^t = R^t - \mmean$ implying $R^t = \xi^t + \mmean$. Thus, we have $\theta_t = \theta_{t-1} - \gamma (\xi^t + \mmean)$. On substituting above, we get:
    \begin{align*}
        \textstyle \mathcal{L}_\mathcal{H}(\theta_t) - \mathcal{L}_\mathcal{H}(\theta_{t-1}) &\leq \textstyle \iprod{-\gamma (\xi^t + \mmean)}{\lossHonest{t-1}} + \frac{L}{2}\norm{\gamma (\xi^t + \mmean)}^2 \\
        &=\textstyle  -\gamma \iprod{\xi^t}{\lossHonest{t-1}} - \gamma \iprod{\mmean}{\lossHonest{t-1}} + \gamma^2\frac{L}{2}\norm{\xi^t + \mmean}^2.
    \end{align*}
    Adding and subtracting $\lossHonest{t-1}$ from $\mmean$ in the second term above yields
    \begin{flalign*}
        \textstyle \mathcal{L}_\mathcal{H}(\theta_t) - \mathcal{L}_\mathcal{H}(\theta_{t-1})&\leq -\gamma \iprod{\xi^t}{\lossHonest{t-1}} - \gamma \iprod{\mmean-\lossHonest{t-1}+\lossHonest{t-1}}{\lossHonest{t-1}}\\
        &\textstyle + \gamma^2\frac{L}{2}\norm{\xi^t + \mmean}^2.
    \end{flalign*}
    Recall from Lemma \ref{lem:mmtDevGlobalG0} that $\delta^t = \mmean - \lossHonest{t-1}$, hence on substituting above we obtain
    \begin{flalign*}
        \textstyle \mathcal{L}_\mathcal{H}(\theta_t) - \mathcal{L}_\mathcal{H}(\theta_{t-1})  &\leq \textstyle -\gamma \iprod{\xi^t}{\lossHonest{t-1}} - \gamma \iprod{\delta^t+\lossHonest{t-1}}{\lossHonest{t-1}}+ \gamma^2\frac{L}{2}\norm{\xi^t + \mmean}^2\\
        &= \textstyle -\gamma \iprod{\xi^t}{\lossHonest{t-1}} - \gamma \iprod{\delta^t}{\lossHonest{t-1}} - \gamma \norm{\lossHonest{t-1}}^2 + \gamma^2\frac{L}{2}\norm{\xi^t + \mmean}^2.
    \end{flalign*}
    Scaling the above equation by $2$ yields:
    \begin{flalign}\label{eq:inter_lem7_1}
        \textstyle 2\mathcal{L}_\mathcal{H}(\theta_t) - 2\mathcal{L}_\mathcal{H}(\theta_{t-1}) &\leq \textstyle -2\gamma \iprod{\xi^t}{\lossHonest{t-1}} - 2\gamma \iprod{\delta^t}{\lossHonest{t-1}} - 2\gamma \norm{\lossHonest{t-1}}^2 + 2\gamma^2\frac{L}{2}\norm{\xi^t + \mmean}^2.
    \end{flalign}
    Using Cauchy-Schwartz inequality, and the fact that $2ab \leq \frac{1}{c}a^2 + cb^2$ for any $c>0$, we obtain the following inequalities:
    \begin{equation}\label{eq:inter_lem7_2}
        \textstyle 2|\iprod{\delta^t}{\lossHonest{t-1}}| \leq 2\norm{\delta^t}\norm{\lossHonest{t-1}} \leq \frac{2}{1}\norm{\delta^t}^2 + \frac{1}{2}\norm{\lossHonest{t-1}}^2.
    \end{equation}
    Similarly, 
    \begin{equation}\label{eq:inter_lem7_3}
        \textstyle 2|\iprod{\xi^t}{\lossHonest{t-1}}| \leq 2\norm{\xi^t}\norm{\lossHonest{t-1}} \leq \frac{2}{1}\norm{\xi^t}^2 + \frac{1}{2}\norm{\lossHonest{t-1}}^2.
    \end{equation}
    Finally, using the triangle inequality, we obtain 
    \begin{flalign*}
        \textstyle \norm{\mmean + \xi^t} &\leq 2\norm{\mmean}^2 + 2\norm{\xi^t}^2 = 2\norm{\mmean-\lossHonest{t-1} + \lossHonest{t-1}}^2 + 2\norm{\xi^t}^2 \nonumber\\
        &\leq 4\norm{\mmean-\lossHonest{t-1}}^2 + 4\norm{\lossHonest{t-1}}^2 + 2\norm{\xi^t}^2.
    \end{flalign*}
    Substituting $\delta^t = \mmean - \lossHonest{t-1}$ above, we get
    \begin{flalign}\label{eq:inter_lem7_4}
        \textstyle \norm{\mmean + \xi^t} &\leq 4\norm{\delta^t}^2 + 4\norm{\lossHonest{t-1}}^2 + 2\norm{\xi^t}^2.
    \end{flalign}
    Substituting (\ref{eq:inter_lem7_2}), (\ref{eq:inter_lem7_3}) and (\ref{eq:inter_lem7_4}) in (\ref{eq:inter_lem7_1}), we get
     \begin{flalign*}
        \textstyle 2\mathcal{L}_\mathcal{H}(\theta_t) - 2\mathcal{L}_\mathcal{H}(\theta_{t-1}) \nonumber &\leq \textstyle \gamma \left(2\norm{\xi^t}^2 + \frac{1}{2}\norm{\lossHonest{t-1}}^2 \right) + \gamma\left(2\norm{\delta^t}^2 + \frac{1}{2}\norm{\lossHonest{t-1}}^2 \right)  \nonumber \\
        &\textstyle - 2\gamma \norm{\lossHonest{t-1}}^2 + 2\gamma^2\frac{L}{2}\left(4\norm{\delta^t}^2 + 4\norm{\lossHonest{t-1}}^2 + 2\norm{\xi^t}^2 \right).
    \end{flalign*}
    Rearranging the R.H.S. in the above yields:
     \begin{flalign*}
        \textstyle 2\mathcal{L}_\mathcal{H}(\theta_t) - 2\mathcal{L}_\mathcal{H}(\theta_{t-1}) \nonumber &\leq \textstyle -\gamma(1-4\gamma L)\norm{\lossHonest{t-1}}^2 + 2\gamma(1+2\gamma L)\norm{\delta^t}^2 + 2\gamma (1+\gamma L)\norm{\xi^t}^2.
    \end{flalign*}
    As $t$ is arbitrarily chosen from $[T]$, taking expectation on both sides above proves the lemma.
\end{proof}

We are now ready to prove Theorem~\ref{thm:rosdhb}, recalled below.

\begin{theorem}
    Under Assumptions~\ref{assumption:smoothness} and~\ref{assumption:gradientGB}, Algorithm~\ref{alg:rghb} with an $(f,\kappa)$-robust aggregation rule $F$ such that $\kappa B^2 \leq \tfrac{1}{7}$, a learning rate $\gamma \leq \tfrac{k}{d c L}$ (with $c = 23200$), and a momentum coefficient $\beta = \sqrt{1 - 24\gamma L}$, satisfies the following guarantee
    \begin{align*}
        \expect{\norm{\nabla \mathcal{L}_\mathcal{H}(\hat{\theta})}^2} \leq \frac{45 \left(\mathcal{L}_\mathcal{H}(\theta^0) - \mathcal{L}_\mathcal{H}^*\right)}{\gamma T ( 1- \kappa B^2)} + \frac{216 \kappa \, G^2}{1 - \kappa B^2},
    \end{align*}
    where $\mathcal{L}_\mathcal{H}^* = \min_{\theta \in \R^d} \mathcal{L}_\mathcal{H}(\theta)$.
\end{theorem}
\begin{proof}
    Fix an arbitrary $t \in [T]$. We define the Lyapunov function for convergence analysis to be
    \begin{equation}\label{eq:GDlyapunovcase2_bounded}
        V^t \coloneqq \expect{2\mathcal{L}_\mathcal{H}(\theta_{t-1}) + z\norm{\delta^t}^2 + z'\Upsilon^{t-1}_\mathcal{H}} ,
    \end{equation}
    where $z = \frac{1}{8L}$ and $z' = \frac{\kappa}{4L}$.

    \medskip
   \noindent Using the upper bound on $\expect{\norm{\delta^{t+1}}^2}$ from Lemma \ref{lem:mmtDevGlobalG0}, we obtain that
    \begin{align}\label{eq:GlobalG0deltadiff_bounded}
        \mathbb{E}[z \norm{ \delta^{t+1}}^2 - z\norm{\delta^t}^2]
        &\leq  \textstyle z\zeta \expect{\norm{\delta^{t}}^2} + 2z\gamma L\left( 2\gamma L (1-\beta)^2\left(\frac{d}{k}-1\right) + (1+\gamma L)\beta^2\right)\expect{\norm{\xi^{t}}^2} \nonumber\\
            &\hspace{10pt}\textstyle + z\left( 2(4\gamma^2L^2+1)(1-\beta)^2\left(\frac{d}{k}-1\right) + 4\gamma L (1+\gamma L)\beta^2\right) \expect{\norm{\lossHonest{t-1}}^2} - z\expect{\norm{\delta^t}^2} \nonumber\\
        &\leq  \textstyle z(\zeta-1) \expect{\norm{\delta^{t}}^2} + 2z\gamma L\left( 2\gamma L (1-\beta)^2\left(\frac{d}{k}-1\right) + (1+\gamma L)\beta^2\right)\expect{\norm{\xi^{t}}^2} \nonumber\\
            &\hspace{10pt}\textstyle + z\left( 2(4\gamma^2L^2+1)(1-\beta)^2\left(\frac{d}{k}-1\right) + 4\gamma L (1+\gamma L)\beta^2\right) \expect{\norm{\lossHonest{t-1}}^2} .
    \end{align}

    \noindent Similarly, using Lemma \ref{lem:mmtDriftGlobalGDG0} we get
    \begin{flalign}\label{eq:tempGlobalupislonz}
        z' \expect{ \Upsilon^t_\mathcal{H}} - z'\expect{\Upsilon^{t-1}_\mathcal{H}} &\leq \textstyle z'\beta \expect{\Upsilon^{t-1}_\mathcal{H}} + z'\left(\scriptstyle{\left(1-\beta\right)^2\frac{d}{k} + \beta(1-\beta)}\right)\left( G^2 + B^2\norm{\lossHonest{t-1}}\right)- z'\expect{\Upsilon^{t-1}_\mathcal{H}}\nonumber\\
         &\leq \textstyle z'(\beta-1) \expect{\Upsilon^{t-1}_\mathcal{H}} + z'\left(\left(1-\beta\right)^2\frac{d}{k} + \beta(1-\beta)\right)\left( G^2 + B^2\norm{\lossHonest{t-1}}\right) .
    \end{flalign}
    Similarly, using Lemma~\ref{lem:modelDriftDevmmtGD}, we get
    \begin{align}\label{eq:tempGloballossDiff}
        \expect{2\mathcal{L}_\mathcal{H}(\theta_{t}) - 2\mathcal{L}_\mathcal{H}(\theta_{t-1})} &\leq -\gamma(1-4\gamma L)\expect{\norm{\lossHonest{t-1}}^2} + 2\gamma(1+2\gamma L)\expect{\norm{\delta^t}^2} + 2\gamma(1+\gamma L)\expect{\norm{\xi^t}^2}
    \end{align}
    
    \noindent Using (\ref{eq:GDlyapunovcase2_bounded}),  (\ref{eq:GlobalG0deltadiff_bounded}), (\ref{eq:tempGlobalupislonz}) and (\ref{eq:tempGloballossDiff}) below:
    \begin{flalign}\label{eq:mmtGDGlobalG0_mainV}
        V^{t+1} - V^t &= \expect{2\mathcal{L}_\mathcal{H}(\theta_{t}) - 2\mathcal{L}_\mathcal{H}(\theta_{t-1})} + \expect{z\norm{\delta^{t+1}}^2 - z\norm{\delta^t}^2} + z' \expect{ \Upsilon^t_\mathcal{H}} - z'\expect{\Upsilon^{t-1}_\mathcal{H}} && \nonumber \\
        &\leq -\gamma(1-4\gamma L)\expect{\norm{\nabla\mathcal{L}_\mathcal{H}(\theta_{t-1})}^2} + 2\gamma(1+2\gamma L)\expect{\norm{\delta^t}^2} + 2\gamma(1+\gamma L)\expect{\norm{\xi^t}^2}  \nonumber\\
            &\hspace{10pt}\textstyle + z(\zeta-1) \expect{\norm{\delta^{t}}^2} + 2z\gamma L\left( 2\gamma L (1-\beta)^2\left(\frac{d}{k}-1\right) + (1+\gamma L)\beta^2\right)\expect{\norm{\xi^{t}}^2} \nonumber\\
            &\hspace{10pt}\textstyle + z\left( 2(4\gamma^2L^2+1)(1-\beta)^2\left(\frac{d}{k}-1\right) + 4\gamma L (1+\gamma L)\beta^2\right) \expect{\norm{\lossHonest{t-1}}^2} \nonumber\\
            &\hspace{10pt}\textstyle + z'(\beta-1) \expect{\Upsilon^{t-1}_\mathcal{H}} + z'\left(\left(1-\beta\right)^2\frac{d}{k} + \beta(1-\beta)\right)\left( G^2 + B^2\norm{\lossHonest{t-1}}\right)\nonumber \\
        &= \textstyle -\gamma\left(1-4\gamma L- \frac{z}{\gamma }\left( 2(4\gamma^2L^2+1)(1-\beta)^2\left(\frac{d}{k}-1\right) + 4\gamma L (1+\gamma L)\beta^2\right) \right.  \nonumber \\
            &\text{\hspace{10pt}}\textstyle \left. -\frac{z'}{\gamma} \left(\left(1-\beta\right)^2\frac{d}{k} + \beta(1-\beta)\right)B^2\right)\expect{\norm{\lossHonest{t}}^2} +\left(2\gamma(1+2\gamma L) + z(\zeta -1)\right)\expect{\norm{\delta^t}^2} \nonumber\\
            &\text{\hspace{10pt}}\textstyle +2\gamma \left( 1+\gamma L+ zL\left( 2\gamma L (1-\beta)^2\left(\frac{d}{k}-1\right) + (1+\gamma L)\beta^2\right)\right)\expect{\norm{\xi^t}^2}  + z'(\beta-1) \expect{\Upsilon^{t-1}_\mathcal{H}}\nonumber \\
            &\textstyle \text{\hspace{10pt}}+ z'\left(\left(1-\beta\right)^2\frac{d}{k} + \beta(1-\beta)\right) G^2
    \end{flalign}
    Denote
    \begin{align}\label{eq:mmtGDGlobalG0ABC}
        A_1 &= \textstyle 1-4\gamma L- \frac{z}{\gamma }\left( 2(4\gamma^2L^2+1)(1-\beta)^2\left(\frac{d}{k}-1\right) + 4\gamma L (1+\gamma L)\beta^2\right) -\frac{z'}{\gamma} \left(\left(1-\beta\right)^2\frac{d}{k} + \beta(1-\beta)\right)B^2 ,\nonumber\\
        A_2 &= \textstyle  2\gamma(1+2\gamma L) + z(\zeta -1)  ,  ~~ \text{ and } \nonumber\\
        A_3 &= \textstyle 1+\gamma L+ zL\left( 2\gamma L (1-\beta)^2\left(\frac{d}{k}-1\right) + (1+\gamma L)\beta^2\right) .
    \end{align}
    Then, we can rewrite (\ref{eq:mmtGDGlobalG0_mainV}) as:
    \begin{align}\label{eq:GDcase2Momentum_mainVA123}
        V^{t+1} - V^t &\leq \textstyle - \gamma A_1 \expect{\norm{\nabla\mathcal{L}_\mathcal{H}(\theta_{t-1})}^2} + A_2\expect{\norm{\delta^t}^2} + 2\gamma A_3\expect{\norm{\xi^t}^2} + z'(\beta-1) \expect{\Upsilon^{t-1}_\mathcal{H}}\nonumber \\
        &\hspace{5pt}\textstyle + z'\left(\left(1-\beta\right)^2\frac{d}{k} + \beta(1-\beta)\right)G^2 .
    \end{align}


    \noindent {\bf Consider term $A_1$}. Since $\beta < 1$, we obtain that 
    \begin{align*}
        A_1 &\geq  \textstyle  1-4\gamma L- \frac{z}{\gamma }\left( 2(4\gamma^2L^2+1)(1-\beta^2)^2\left(\frac{d}{k}-1\right) + 4\gamma L (1+\gamma L)\beta^2\right) -\frac{z'}{\gamma} \left(\left(1-\beta^2\right)^2\frac{d}{k} + \frac{1}{2}(1-\beta^2)\right)B^2.
    \end{align*}
    Substituting $1 - \beta^2 = 24\gamma L$ above, we obtain that
    \begin{align*}
        A_1 
            &\geq \textstyle 1-4\gamma L- \frac{z}{\gamma }\left( 2(4\gamma^2L^2+1)(24\gamma L)^2\left(\frac{d}{k}-1\right) + 4\gamma L (1+\gamma L)(1- 24 \gamma L)\right) -\frac{z'}{\gamma} (24\gamma L)^2\frac{d}{k}B^2 - \frac{z'}{\gamma}12\gamma L)B^2 \\
            &= \textstyle 1-4\gamma L- 2\gamma z(4\gamma^2L^2+1)(24L)^2\left(\frac{d}{k}-1\right) - 4zL (1+\gamma L)(1 -24 \gamma L) - 576z'\gamma L^2\frac{d}{k}B^2 - 12 z' L B^2 .
    \end{align*}
    Substituting $z = \frac{1}{8L}$ and $z' = \frac{\kappa}{4L}$ above yields, 
    

    \begin{flalign}
        A_1 
            &\geq \textstyle 1-4\gamma L- \frac{2\gamma}{8L} (4\gamma^2L^2+1)(24L)^2\left(\frac{d}{k}-1\right) - \frac{4L}{8L} (1+\gamma L)+12 \gamma L + 12 (\gamma L)^2- 576\left(\frac{\kappa}{4L} \right)\gamma L^2\frac{d}{k}B^2 - 12 \left(\frac{\kappa}{4L} \right) L B^2 \nonumber \\
            &= \textstyle \frac{1}{2}+\frac{15}{2}\gamma L + 12 (\gamma L )^2 -  144\gamma L(4\gamma^2L^2+1)\left(\frac{d}{k}-1\right) - 144\kappa\gamma L\frac{d}{k}B^2 - 3\kappa B^2 \nonumber \\ 
            & \ge \textstyle \frac{1}{2}-  144\gamma L(4\gamma^2L^2+1)\left(\frac{d}{k}-1\right) - 144\kappa\gamma L\frac{d}{k}B^2 - 3\kappa B^2. \nonumber
    \end{flalign}

    Recall that $\gamma \leq \frac{k/d}{c L}$ where $c = 23200$. Moreover, we assume that $\kappa \leq \frac{1}{7 B^2}$.
    Using this above yields, 
    \begin{align}\label{eq:valAglobalGDGBmmt}
        A_1 &\geq \textstyle \frac{1}{2} -  144L\left(\frac{d}{k}-1\right)\times \frac{k/d}{23200L}\left(1+\frac{(k/d)^2}{23200^2}\right) - 144 L\frac{d}{k}B^2 \times \frac{k/d}{23200 L} \times \frac{1}{7 B^2} - \frac{3 B^2}{7 B^2} \nonumber \\
         &= \textstyle \frac{1}{2} -  \frac{144}{23200}\left(1-\frac{k}{d}\right)\times \left(1+\frac{(k/d)^ 2}{23200^2}\right) - \frac{144}{23200}\times \frac{1}{7}  - \frac{3}{7} \nonumber \\
         &\ge \textstyle \frac{1}{2} - \frac{1}{160} \left(1+\frac{1}{23200^2}\right)-\frac{1}{160 \times 7} - \frac{3}{7} 
         > \frac{1}{18}.
    \end{align}

    \medskip
    
    \noindent {\bf Consider term $A_2$}. Substituting $\zeta = 8\gamma^2 L^2 (1-\beta)^2\left(\frac{d}{k}-1\right) + (1+4\gamma L)(1+\gamma L)\beta^2$, we obtain that 
    \begin{align*}
        A_2 & = 2\gamma(1+2\gamma L) + z(\zeta -1) \\
        & = \textstyle 2\gamma(1+2\gamma L) + z\left( 8\gamma^2 L^2 (1-\beta)^2\left(\frac{d}{k}-1\right) + (1+4\gamma L)(1+\gamma L)\beta^2 -1 \right) \\
        & = \textstyle 2\gamma(1+2\gamma L) - z (1 - \beta^2) + 8z\gamma^2 L^2 (1-\beta^2)^2\left(\frac{d}{k}-1\right) + z \gamma L (5 + 4\gamma L)\beta^2 .
    \end{align*}
    Since $\beta < 1$, from above we obtain that
    \begin{align*}
        A_2 & \leq 
            \textstyle 2\gamma(1+2\gamma L) - z(1- \beta^2)  + 8z\gamma^2 L^2 (1-\beta^2)^2\left(\frac{d}{k}-1\right) + z \gamma L (5 + 4\gamma L) .
    \end{align*}
    \noindent Substituting $1 - \beta^2 = 24\gamma L$ above, we obtain that
    \begin{align*}
        A_2&\leq \textstyle  - z(24\gamma L)  + 8z\gamma^2 L^2 (24\gamma L)^2\left(\frac{d}{k}-1\right) + z\gamma L(5 + 4\gamma L) + 2\gamma(1+2\gamma L) .
    \end{align*}
    Substituting $z = \frac{1}{8L}$ above yields, 
    \begin{align*}
        A_2 &\leq \textstyle  - \frac{24\gamma L}{8L}  + \gamma\left(\frac{8\gamma L^2}{8L} (24\gamma L)^2\left(\frac{d}{k}-1\right) + \frac{L}{8L}(5+ 4\gamma L) + 2(1+2\gamma L)\right) \\
            &\leq \textstyle  - 3\gamma  + \gamma\left(\gamma L(24\gamma L)^2\left(\frac{d}{k}-1\right) + \frac{1}{8}(5 + 4\gamma L ) + 2(1+2\gamma L)\right) .
    \end{align*}
    Since $\gamma \leq \frac{1}{24L}$, from above we obtain that
    \begin{align*}
        A_2 &\leq \textstyle  - 3\gamma  + \gamma\left(\gamma L(\frac{24 L}{24L})^2\left(\frac{d}{k}-1\right) + \frac{1}{8}(5 + \frac{4 L}{24L} ) + 2(1+ \frac{2 L}{24L})\right) \\
            &\leq \textstyle  - 3\gamma  + \gamma\left(\gamma L\left(\frac{d}{k}-1\right) + \frac{1}{8}(\frac{31}{6} ) + \frac{13}{6}\right) \\
            &= \textstyle  - 3\gamma  + \gamma\left(\left(\frac{d}{k}-1\right)\gamma L + \frac{45}{16}\right) .
    \end{align*}
    Since $\gamma \leq \frac{k/d}{23200L}$, the above yields,
    \begin{align}\label{eq:valBglobalGDGBmmt}
        A_2 &\leq \textstyle - 3\gamma  + \gamma\left(\left(\frac{d}{k}-1\right)L \times \frac{k/d}{23200L} + \frac{45}{16}\right)  < - 3\gamma + \frac{46}{16}\gamma < 0 .
    \end{align}

    \medskip

    \noindent {\bf Consider term $A_3$}. Since $\beta < 1$, we obtain that 
    \begin{align*}
        A_3 &= \textstyle  1+\gamma L+ zL\left( 2\gamma L (1-\beta^2)^2\frac{d}{k} + (1+\gamma L)\beta^2\right) \\ 
        &\leq \textstyle  1+\gamma L+ zL\left( 2\gamma L (1-\beta^2)^2\frac{d}{k} + (1+\gamma L)\right) . 
    \end{align*}
    Substituting $1-\beta^2 = 24\gamma L$, we obtain that
    \begin{align*}
        A_3 &\leq \textstyle 1+\gamma L+ zL\left( 2\gamma L (24\gamma L)^2\left(\frac{d}{k}-1\right) + (1+\gamma L)\right) .
    \end{align*}
    Substituting $z = \frac{1}{8L}$ above, we get
    \begin{align*}
        A_3 &\leq \textstyle 1+\gamma L+ \frac{L}{8L}\left( 2\gamma L (24\gamma L)^2\left(\frac{d}{k}-1\right) + (1+\gamma L)\right) .
    \end{align*}
    Since $\gamma \leq \frac{1}{24L}$, from above we obtain that
    \begin{align*}
        A_3 &\leq \textstyle 1+ \frac{L}{24L}+  \frac{\gamma L}{4} (\frac{24L}{24L})^2\left(\frac{d}{k}-1\right) + \frac{1}{8}(1+ \frac{L}{24L})  = \textstyle \frac{25}{24}+ \frac{1}{4}\gamma L \left(\frac{d}{k}-1\right)+ \frac{1}{8}(\frac{25}{24})\\
            &\leq \textstyle \frac{75}{64} + \frac{1}{4}\left(\frac{d}{k}-1\right)\gamma L .
    \end{align*}
    Recalling that $\gamma \leq \frac{k/d}{23200L}$, we obtain,
    \begin{align}\label{eq:valCglobalGDGBmmt}
        A_3 &\leq \textstyle \frac{75}{64} + \frac{1}{4}\left(\frac{d}{k}-1\right)L \times \frac{k/d}{23200L} = \frac{75}{64} + \frac{1}{92800}\left( 1-\frac{k}{d}\right)  < \frac{76}{64} < \frac{3}{2}.
    \end{align}
    \noindent Substituting values from (\ref{eq:valAglobalGDGBmmt}), (\ref{eq:valBglobalGDGBmmt}) and (\ref{eq:valCglobalGDGBmmt}) in (\ref{eq:GDcase2Momentum_mainVA123}): 
    \begin{flalign*}
        V^{t+1} - V^t 
        &\leq  \textstyle - \frac{\gamma}{18} \expect{\norm{\nabla\mathcal{L}_\mathcal{H}(\theta_{t-1})}^2} + 3\gamma \expect{\norm{\xi^t}^2} + z'(\beta-1) \expect{\Upsilon^{t-1}_\mathcal{H}}\\
        &\textstyle + z'\left(\left(1-\beta\right)^2\frac{d}{k} + \beta(1-\beta)\right) G^2 .
    \end{flalign*}
    \noindent Substituting value of $\expect{\norm{\xi^t}^2}$ from Lemma \ref{lem:MmtDriftGlobalG0} above, we get
    \begin{flalign*}
        V^{t+1} - V^t &\leq \textstyle - \frac{\gamma}{18} \expect{\norm{\nabla\mathcal{L}_\mathcal{H}(\theta_{t-1})}^2} + 3\gamma \kappa \Bigl( \beta \expect{\Upsilon^{t-1}_\mathcal{H}} + \left(\left(1-\beta\right)^2\frac{d}{k} + \beta(1-\beta)\right)\Bigl(G^2  \\
            &\hspace{10pt}\textstyle + B^2 \expect{\norm{ \lossHonest{t-1}}^2}\Bigr)\Bigr) + z'(\beta-1) \expect{\Upsilon^{t-1}_\mathcal{H}} + z'\left(\left(1-\beta\right)^2\frac{d}{k} + \beta(1-\beta)\right) G^2 \\
        &\leq  \textstyle - \left(\frac{\gamma}{18} - 3 \gamma \kappa \left(\left(1-\beta\right)^2\frac{d}{k} + \beta(1-\beta)\right) B^2 \right)\expect{\norm{\nabla\mathcal{L}_\mathcal{H}(\theta_{t-1})}^2} + (3\gamma \kappa + z') \\
            &\hspace{10pt} \textstyle \times \Bigl(\bigl(1-\beta\bigr)^2\frac{d}{k} + \beta(1-\beta)\Bigr) G^2 + \left( 3\gamma \kappa \beta - (1-\beta)z' \right) \expect{\Upsilon^{t-1}_\mathcal{H}} .
    \end{flalign*}
    Since $1-\beta = \frac{1-\beta^2}{1+\beta}$, we have $\frac{1-\beta^2}{2} \leq 1-\beta \leq 1-\beta^2$, resulting in
    \begin{flalign*}
        V^{t+1} - V^t &\leq  \textstyle -\frac{\gamma}{18} \left(1 - 54 \kappa \left(\left(1-\beta^2\right)^2\frac{d}{k} + \beta(1-\beta^2)\right) B^2 \right)\expect{\norm{\nabla\mathcal{L}_\mathcal{H}(\theta_{t-1})}^2}  \\
            &\hspace{10pt} \textstyle + \left( 3\gamma \kappa \beta - \frac{(1-\beta^2)}{2}z' \right) \expect{\Upsilon^{t-1}_\mathcal{H}} + ( 3\gamma \kappa  + z' )\left(\left(1-\beta^2\right)^2\frac{d}{k} + \beta(1-\beta^2)\right) G^2. 
    \end{flalign*}
    Substituting $1-\beta^2 = 24\gamma L$ above, we obtain
    \begin{flalign*}
        V^{t+1} - V^t &\leq  \textstyle -\frac{\gamma}{18} \left(1 - 54\kappa \left(\left(24\gamma L\right)^2\frac{d}{k} + 24\gamma L\right) B^2 \right)\expect{\norm{\nabla\mathcal{L}_\mathcal{H}(\theta_{t-1})}^2} \\
            &\hspace{10pt} \textstyle + \left(3\gamma \kappa + z'\right)\left(\left(24\gamma L\right)^2\frac{d}{k} + 24\gamma L\right) G^2  + \left( 3\gamma \kappa - \frac{24\gamma}{2} z' L \right) \expect{\Upsilon^{t-1}_\mathcal{H}} \\
        & = \textstyle -\frac{\gamma}{18} \left(1 - 54\kappa \left(\left(24\gamma L\right)^2\frac{d}{k} + 24\gamma L\right) B^2 \right)\expect{\norm{\nabla\mathcal{L}_\mathcal{H}(\theta_{t-1})}^2} \\
            &\hspace{10pt} \textstyle + \left(3\gamma \kappa + z'\right)\left(576\gamma^2L^2\frac{d}{k} + 24\gamma L\right) G^2  + \left( 3\gamma \kappa - 12 \gamma z' L \right) \expect{\Upsilon^{t-1}_\mathcal{H}}. 
    \end{flalign*}

    \noindent Since $\gamma \leq \frac{1}{24L}$, from above we obtain that
    \begin{flalign*}
        V^{t+1} - V^t &\leq  \textstyle -\frac{\gamma}{18} \left(1 - 54\kappa \left(\left(\frac{24 L}{24L}\times 24\gamma L\right)\frac{d}{k} + 24\gamma L\right) B^2 \right)\expect{\norm{\nabla\mathcal{L}_\mathcal{H}(\theta_{t-1})}^2} \\
            &\hspace{10pt} \textstyle  + \left( 3\gamma \kappa - 12 z'\gamma L \right) \expect{\Upsilon^{t-1}_\mathcal{H}} + \left(3\kappa \gamma + z'\right)\left(576\gamma^2L^2\frac{d}{k} + 24\gamma L\right) G^2\\
        &= \textstyle -\frac{\gamma}{18} \left(1 - 54\kappa \left( 24\gamma L\frac{d}{k} + 24\gamma L\right) B^2 \right)\expect{\norm{\nabla\mathcal{L}_\mathcal{H}(\theta_{t-1})}^2}  \\
            &\hspace{10pt} \textstyle + \left( 3\gamma \kappa - 12 z'\gamma L \right) \expect{\Upsilon^{t-1}_\mathcal{H}} + \left(3\kappa \gamma + z'\right)\left(576\gamma^2L^2\frac{d}{k} + 24\gamma L\right) G^2 .
    \end{flalign*}
    
    \noindent Substituting $z' = \frac{\kappa}{4L}$ above, we obtain that
    \begin{flalign*}
        V^{t+1} - V^t 
        &\leq \textstyle -\frac{\gamma}{18} \left(1 - 54\kappa \left( 24\gamma L\frac{d}{k} + 24\gamma L\right) B^2 \right)\expect{\norm{\nabla\mathcal{L}_\mathcal{H}(\theta_{t-1})}^2} &&\\
            &\hspace{10pt} \textstyle  + \left( 3\gamma \kappa - \frac{12L\kappa}{4L}\gamma \right) \expect{\Upsilon^{t-1}_\mathcal{H}} + \left(3\kappa \gamma + \frac{\kappa}{4L}\right)\left(576\gamma^2L^2\frac{d}{k} + 24\gamma L\right) G^2\\
        &=\textstyle -\frac{\gamma}{18} \left(1 - 54\kappa \left( 24\gamma L\frac{d}{k} + 24\gamma L\right) B^2 \right)\expect{\norm{\nabla\mathcal{L}_\mathcal{H}(\theta_{t-1})}^2}  \\
            &\hspace{10pt} \textstyle + \left( 3\gamma \kappa - 3\kappa \gamma \right) \expect{\Upsilon^{t-1}_\mathcal{H}} + \frac{\kappa}{4L}\left(12 \gamma L + 1\right)\left(576\gamma^2L^2\frac{d}{k} + 24\gamma L\right) G^2\\
        &= \textstyle -\frac{\gamma}{18} \left(1 - 54\kappa \left( 24\gamma L\frac{d}{k} + 24\gamma L\right) B^2 \right)\expect{\norm{\nabla\mathcal{L}_\mathcal{H}(\theta_{t-1})}^2} + \kappa\left(12 \gamma L+1\right)\bigl(144\gamma^2L\frac{d}{k}+ 6\gamma\bigr) G^2.
    \end{flalign*}
    
    \noindent Substituting $\gamma \leq \frac{k/d}{23200L} \leq \frac{1}{23200L} $ yields,
    \begin{flalign*}
        V^{t+1} - V^t &\leq \textstyle -\frac{\gamma}{18} \left(1 - 54\kappa \left( 24L\frac{d}{k}\times \frac{k/d}{23200L}  + \frac{24L}{23200L}\right) B^2 \right)\expect{\norm{\nabla\mathcal{L}_\mathcal{H}(\theta_{t-1})}^2}&& \\
            &\hspace{10pt} \textstyle  + \kappa\left(12 \gamma L + 1\right)\left(144\gamma^2 L\frac{d}{k}+ 6\gamma\right) G^2\\
        &< \textstyle -\frac{\gamma}{18} \left(1 - \frac{54 \times 48}{23200 }\kappa B^2\right)\expect{\norm{\nabla\mathcal{L}_\mathcal{H}(\theta_{t-1})}^2} + 6\kappa\gamma\left(288\gamma^2 L ^2\frac{d}{k} + 24\gamma L \frac{d}{k} + 12\gamma L + 1 \right) G^2.
    \end{flalign*}
    \noindent 
    Note that $\left(1 - \frac{54 \times 48}{23200 }\kappa B^2\right) \geq 1 - \kappa B^2$, which is positive since $7\kappa B^2 \leq 1$.  Using this above yields,

    \begin{flalign*}
        V^{t+1} - V^t 
        &\leq \textstyle -\frac{\gamma}{18} ( 1- \kappa B^2)\expect{\norm{\nabla\mathcal{L}_\mathcal{H}(\theta_{t-1})}^2} + 6\kappa\gamma\left(288\gamma^2 L ^2\frac{d}{k} + 24\gamma L \frac{d}{k} + 12\gamma L + 1 \right) G^2.
    \end{flalign*}
    
    \noindent Moving the term $\lossHonest{t-1}$ to the L.H.S, we get
    \begin{flalign*}
       \textstyle  \frac{\gamma}{18} ( 1- \kappa B^2) \expect{\norm{\nabla\mathcal{L}_\mathcal{H}(\theta_{t-1})}^2}
            &\leq \textstyle V^t  - V^{t+1} + 6\kappa\gamma\left(288\gamma^2 L ^2\frac{d}{k} + 24\gamma L \frac{d}{k} + 12\gamma L + 1 \right) G^2.
    \end{flalign*}
    
    \noindent Taking summation on both the sides from $t=1$ to $T$ and rearranging, we get
    \begin{align*}
        \textstyle \frac{\gamma ( 1- \kappa B^2)}{18} \sum\limits_{t=1}^T \expect{\norm{\nabla\mathcal{L}_\mathcal{H}(\theta_{t-1})}^2} &\leq \textstyle V^1 - V^{T+1} + 6\kappa\gamma\left(288\gamma^2 L ^2\frac{d}{k} + 24\gamma L \frac{d}{k} + 12\gamma L + 1 \right) G^2 T.
    \end{align*}
    
    \noindent Multiplying both sides by $\frac{18}{\gamma ( 1- \kappa B^2) }$,
    \begin{align*}
        \textstyle \sum_{t=1}^T \expect{\norm{\nabla\mathcal{L}_\mathcal{H}(\theta_{t-1})}^2} &\leq {\textstyle  \frac{18\left(V^1 - V^{T+1}\right)}{\gamma ( 1- \kappa B^2)} + 108 \kappa \left(288\gamma^2 L ^2\frac{d}{k} + 24\gamma L \frac{d}{k} + 12\gamma L + 1 \right) \frac{G^2 T}{1 - \kappa B^2}}.
    \end{align*}
    
    \noindent Dividing both sides by $T$, we get,
    \begin{align}\label{eq:GDavg_loss_case2bounded}
        \textstyle \frac{1}{T} \sum\limits_{t=1}^T \expect{\norm{\nabla\mathcal{L}_\mathcal{H}(\theta_{t-1})}^2} &\leq {\textstyle  \frac{18\left(V^1 - V^{T+1}\right)}{\gamma T ( 1- \kappa B^2)} + 108 \kappa \left(288\gamma^2 L ^2\frac{d}{k} + 24\gamma L \frac{d}{k} + 12\gamma L + 1 \right) \frac{G^2}{1 - \kappa B^2}}.
    \end{align}
    
    Next, we obtain an upper bound on $V^1 - V^{T+1}$. Let $\mathcal{L}^*_\mathcal{H} = \inf_{\theta\in\mathbb{R}^d} \mathcal{L}_\mathcal{H}(\theta)$. For any $t>0$, using definition of $V^t$ from (\ref{eq:GDlyapunovcase2_bounded}), we have:
    \begin{align*}
        V^t - 2\mathcal{L}^*_\mathcal{H} &= 2\expect{\mathcal{L}_\mathcal{H}(\theta_{t-1}) - \mathcal{L}^*_\mathcal{H}} + z\expect{\norm{\delta^t}^2} \geq 0 + z\expect{\norm{\delta^t}^2} + z'\expect{\Upsilon^{t-1}_\mathcal{H}} \geq 0.
    \end{align*}
    
    \noindent Thus, $V^t \geq 2\mathcal{L}^*_\mathcal{H}$, implying:
    \begin{align}\label{eq:GDVdiff2_Tbounded}
         V^1 - V^{T+1} \leq V^1 - 2\mathcal{L}^*_\mathcal{H} .
    \end{align}
    
    \noindent Note that for $t=1$, from~\eqref{eq:GDlyapunovcase2_bounded}, and the definitions of $\delta^t$ and $\Upsilon^{0}_\mathcal{H}$ in Lemmas~\ref{lem:mmtDevGlobalG0} and~\ref{lem:mmtDriftGlobalGDG0}, respectively, we have
    \begin{flalign*}
        V^1 &= 2\mathcal{L}_\mathcal{H}(\theta_0) + z\expect{\norm{\delta^1}^2} + z'\expect{\Upsilon^{0}_\mathcal{H}} \\
         &\leq  \textstyle 2\mathcal{L}_\mathcal{H}(\theta_0) + z\expect{\norm{\overline{m}_\mathcal{H}^1 - \nabla\mathcal{L}_\mathcal{H}(\theta_{0})}^2} + z' \frac{1}{\mathcal{H}}\sum_{i\in\mathcal{H}} \condexpect{k}{\norm{\tilde{m}^0_i}^2}\\
        &\leq \textstyle  2\mathcal{L}_\mathcal{H}(\theta_0) + z\expect{\norm{\frac{1}{\mathcal{H}}\sum_{i\in\mathcal{H}}m_i^1 - \nabla\mathcal{L}_\mathcal{H}(\theta_{0})}^2} + z' \frac{1}{\mathcal{H}}\sum_{i\in\mathcal{H}} \condexpect{k}{\norm{\mmmtdrifti{i}{0}}^2}. 
    \end{flalign*}

    \noindent Since $m_i^0 = 0$, $\mmmtdrifti{i}{0} \coloneqq m_i^t - \mmean = 0$ for all $i \in \mathcal{H}$. Using this above, we obtain that
    \begin{flalign*}
        V^1 &\leq \textstyle 2\mathcal{L}_\mathcal{H}(\theta_0) + z\expect{\norm{\overline{m}_\mathcal{H}^1 - \nabla\mathcal{L}_\mathcal{H}(\theta_{0})}^2} = \textstyle 2\mathcal{L}_\mathcal{H}(\theta_0) + z\expect{\norm{\frac{1}{\mathcal{H}}\sum_{i\in\mathcal{H}}m_i^1 - \nabla\mathcal{L}_\mathcal{H}(\theta_{0})}^2}. 
    \end{flalign*}
    
    \noindent Recall that $m_i^t = \beta m_i^{t-1} + (1-\beta)\tilde{g}_i^t$ where $m_i^0 = 0$. Thus, $m_i^1 = (1-\beta)\tilde{g}_i^1$. Using this above, we get
    \begin{flalign}\label{eq:interm_V1}
        V^1
        &\leq \textstyle 2\mathcal{L}_\mathcal{H}(\theta_0) + z\expect{\norm{\frac{1}{\mathcal{H}}\sum_{i\in\mathcal{H}}(1-\beta)\tilde{g}_i^1 - \nabla\mathcal{L}_\mathcal{H}(\theta_{0})}^2} \nonumber \\
        &=  2\mathcal{L}_\mathcal{H}(\theta_0) + z\expect{\norm{(1-\beta)\overline{\tilde{g}}_\mathcal{H}^1 - \nabla\mathcal{L}_\mathcal{H}(\theta_{0})}^2} \nonumber \\
        & \textstyle= 2\mathcal{L}_\mathcal{H}(\theta_0) + z\expect{\norm{(1-\beta)\left(\overline{\tilde{g}}_\mathcal{H}^1 - \nabla\mathcal{L}_\mathcal{H}(\theta_{0})\right
        ) + \beta\nabla\mathcal{L}_\mathcal{H}(\theta_{0})}^2}. 
    \end{flalign}
    
    \noindent Note that
    \begin{align*} \textstyle \expect{\overline{\tilde{g}}_\mathcal{H}^1} &= \condexpect{t}{\condexpect{k}{{\overline{\tilde{g}}_\mathcal{H}^1}}} = \condexpect{t}{\overline{g}^1_\mathcal{H}} = \condexpect{t}{\frac{1}{|\mathcal{H}|} \sum
    \limits_{i\in \mathcal{H}} g_i^1} = \frac{1}{|\mathcal{H}|}\sum\limits
    _{i\in \mathcal{H}} \condexpect{t}{g_i^1} = \frac{1}{|\mathcal{H}|}\sum\limits
    _{i\in \mathcal{H}} \lossWorker{i}{0} = \nabla\mathcal{L}_\mathcal{H}(\theta_{0}). 
    \end{align*}
    Therefore,
    \begin{align*}
        &\textstyle \expect{\norm{(1-\beta)\left(\overline{\tilde{g}}_\mathcal{H}^1 - \nabla\mathcal{L}_\mathcal{H}(\theta_{0})\right) + \beta\nabla\mathcal{L}_\mathcal{H}(\theta_{0})}^2} = \expect{\norm{(1-\beta)\left(\overline{\tilde{g}}_\mathcal{H}^1 - \nabla\mathcal{L}_\mathcal{H}(\theta_{0})\right)}^2
        + \norm{\beta\nabla\mathcal{L}_\mathcal{H}(\theta_{0})}^2}.
    \end{align*}
    
    \noindent Using above in (\ref{eq:interm_V1}), we get:  
    \begin{flalign*}
        V^1
        &\leq \textstyle 2\mathcal{L}_\mathcal{H}(\theta_0) + z\left((1-\beta)^2\expect{\norm{\overline{\tilde{g}}_\mathcal{H}^1 - \nabla\mathcal{L}_\mathcal{H}(\theta_{0})}^2} + \beta^2\expect{\norm{\lossHonest{0}}^2}\right).
    \end{flalign*}
    
   \noindent Since $\expect{\norm{\overline{\tilde{g}}_\mathcal{H}^t - \nabla\mathcal{L}_\mathcal{H}(\theta_{t-1})}^2} \leq \left(\frac{d}{k}-1\right)\expect{\norm{\lossHonest{t-1}}^2}$, owing to Lemma \ref{lem:helperMmtGDGlobalG0}, we get:
   \begin{flalign*}
        V^1
        &\leq \textstyle 2\mathcal{L}_\mathcal{H}(\theta_0) + z\left((1-\beta)^2\left(\frac{d}{k}-1\right)\expect{\norm{\lossHonest{0}}^2} + \beta^2\expect{\norm{\lossHonest{0}}^2}\right)\\
        & = \textstyle 2\mathcal{L}_\mathcal{H}(\theta_0) + z\left((1-\beta^2)^2\left(\frac{d}{k}-1\right)+\beta^2\right)\expect{\norm{\lossHonest{0}}^2}.
    \end{flalign*}
    
    \noindent Recall that $\beta^2 < 1$, $1-\beta^2 = 24\gamma L$ and $z = \frac{1}{8L}$. Therefore, 
    \begin{flalign*}
        V^1
        &\leq \textstyle 2\mathcal{L}_\mathcal{H}(\theta_0) + \frac{1}{8L}\left((24\gamma L)^2\left(\frac{d}{k}-1\right) + 1\right)\expect{\norm{\lossHonest{0}}^2}\\
        & =  \textstyle 2\mathcal{L}_\mathcal{H}(\theta_0) + (72\gamma^2 L\left(\frac{d}{k}-1\right) + \frac{1}{8L}) \expect{\norm{\lossHonest{0}}^2}.
    \end{flalign*}

    \noindent Due to the $L$-Lipschitz smoothness (cf.~Assumption~\ref{assumption:smoothness}) 
    we have $\expect{\norm{\lossHonest{}}^2} \leq 2L\left(\mathcal{L}_\mathcal{H}(\theta) - \mathcal{L}_\mathcal{H}^*\right)$ (See \cite{nesterov2018lectures}, Theorem 2.1.5). Substituting above, we get

    \begin{flalign*}
        V^1
        &\leq \textstyle 2\mathcal{L}_\mathcal{H}(\theta_0) + (72\gamma^2 L\left(\frac{d}{k}-1\right) + \frac{1}{8L})2L\left(\mathcal{L}_\mathcal{H}(\theta_0) - \mathcal{L}_\mathcal{H}^*\right)\\
        &\leq \textstyle 2\mathcal{L}_\mathcal{H}(\theta_0) + (144\gamma^2 L^2\left(\frac{d}{k}-1\right) + \frac{1}{4})\left(\mathcal{L}_\mathcal{H}(\theta_0) - \mathcal{L}_\mathcal{H}^*\right).
    \end{flalign*}
   
   \noindent Substituting from above in (\ref{eq:GDVdiff2_Tbounded}), we obtain that
    \begin{align*}
        V^1 - V^{T+1} &\leq {\textstyle 2 (\mathcal{L}_\mathcal{H}(\theta_0)-\mathcal{L}^*) + (144\gamma^2 L^2\left(\frac{d}{k}-1\right) + \frac{1}{4})\left(\mathcal{L}_\mathcal{H}(\theta_0) - \mathcal{L}_\mathcal{H}^*\right)}\\
        &\leq {\textstyle (144\gamma^2 L^2\left(\frac{d}{k}-1\right) + \frac{9}{4})\left(\mathcal{L}_\mathcal{H}(\theta_0) - \mathcal{L}_\mathcal{H}^*\right)}.
    \end{align*}
    
    \noindent Recall that $\gamma \leq \frac{k/d}{23200L} \leq \frac{1}{23200L}$. Substituting above, we get
    \begin{align*}
        V^1 - V^{T+1} &\leq {\textstyle (\frac{144 L^2\times k/d}{(23200)^2 L^2}\left(\frac{d}{k}-1\right) + \frac{9}{4})\left(\mathcal{L}_\mathcal{H}(\theta_0) - \mathcal{L}_\mathcal{H}^*\right)} = {\textstyle (\frac{144 L^2}{(23200)^2 L^2}\left(1 - \frac{k}{d}\right) + \frac{9}{4})\left(\mathcal{L}_\mathcal{H}(\theta_0) - \mathcal{L}_\mathcal{H}^*\right)}\\
        &\leq {\textstyle \frac{10}{4}\left(\mathcal{L}_\mathcal{H}(\theta_0) - \mathcal{L}_\mathcal{H}^*\right)} = {\textstyle \frac{5}{2}\left(\mathcal{L}_\mathcal{H}(\theta_0) - \mathcal{L}_\mathcal{H}^*\right)}.
    \end{align*}
    
    \noindent Substituting from above in (\ref{eq:GDavg_loss_case2bounded}) proves the theorem, i.e., we obtain that
    \begin{flalign*}
        \textstyle \frac{1}{T}\sum_{t=1}^T \expect{\norm{\nabla\mathcal{L}_\mathcal{H}(\theta_{t-1})}^2} &\leq  {\textstyle  \frac{18 \times \frac{5}{2} \left(\mathcal{L}_\mathcal{H}(\theta_0) - \mathcal{L}_\mathcal{H}^*\right)}{\gamma T ( 1- \kappa B^2)} + 108 \kappa \left(288\gamma^2 L ^2\frac{d}{k} + 24\gamma L \frac{d}{k} + 12\gamma L + 1 \right) \frac{G^2}{1 - \kappa B^2}} \nonumber\\
        &= \textstyle  \frac{45 \left(\mathcal{L}_\mathcal{H}(\theta_0) - \mathcal{L}_\mathcal{H}^*\right)}{\gamma T ( 1- \kappa B^2)} + 108 \kappa \left(288\gamma^2 L ^2\frac{d}{k} + 24\gamma L \frac{d}{k} + 12\gamma L + 1 \right) \frac{G^2}{1 - \kappa B^2}.
    \end{flalign*}
    Note that since $\gamma \leq \frac{k/d}{23200L} \leq \frac{1}{23200L}$, 
    \begin{align*}
        \textstyle 288\gamma^2 L ^2\frac{d}{k} + 24\gamma L \frac{d}{k} + 12\gamma L &\leq \textstyle \frac{288L}{23200} \gamma + \frac{24}{23200} + 12 \gamma L \leq \frac{288}{23200^2}  + \frac{24}{23200} + \frac{12}{23200}  < 1
    \end{align*}
    Using this above, we obtain that
    \begin{flalign*}
        \textstyle \frac{1}{T}\sum_{t=1}^T \expect{\norm{\nabla\mathcal{L}_\mathcal{H}(\theta_{t-1})}^2} &\leq \textstyle  \frac{45 \left(\mathcal{L}_\mathcal{H}(\theta_0) - \mathcal{L}_\mathcal{H}^*\right)}{\gamma T ( 1- \kappa B^2)} + \frac{216\kappa G^2}{1 - \kappa B^2}.
    \end{flalign*}
    Since, $\hat{\theta}$ is chosen uniformly at random from $\{ \theta_0, \ldots, \theta_{T-1}\}$, $\expect{\norm{\nabla\mathcal{L}_\mathcal{H}(\hat{\theta})}^2} = \frac{1}{T}\sum_{t=1}^T \expect{\norm{\nabla\mathcal{L}_\mathcal{H}(\theta_{t-1})}^2}$. Using this above proves the theorem.
\end{proof}

\begin{corollary}
In the same conditions of Theorem~\ref{thm:rosdhb},  let $\gamma = \tfrac{1}{\alpha c L}$, then
        \begin{align*}            
         \textstyle   \expect{\norm{\nabla \mathcal{L}_\mathcal{H}(\hat{\theta})}^2} \leq \mathcal{O}\! \left( \frac{\alpha}{T ( 1- \kappa B^2)} + \frac{\kappa G^2}{1 - \kappa B^2} \right).
        \end{align*}
\end{corollary}

\begin{proof}
    Recall from Theorem \ref{thm:rosdhb} that
    $$\textstyle \expect{\norm{\nabla \mathcal{L}_\mathcal{H}(\hat{\theta})}^2} \leq \frac{45\left(\mathcal{L}_\mathcal{H}(\theta_0) - \mathcal{L}_\mathcal{H}^*\right)}{\gamma T ( 1- \kappa B^2)} + \frac{216 \kappa G^2}{1 - \kappa B^2}.$$

    \noindent Also, for $\alpha = \frac{d}{k}$, recall that $\gamma = \frac{1}{c\alpha L}$ implying that $\frac{1}{\gamma} = c\alpha L$, which results in
    $$\textstyle \expect{\norm{\nabla \mathcal{L}_\mathcal{H}(\hat{\theta})}^2} \leq \frac{45cL \alpha \left(\mathcal{L}_\mathcal{H}(\theta_0) - \mathcal{L}_\mathcal{H}^*\right)}{T ( 1- \kappa B^2)} + \frac{216 \kappa G^2}{1 - \kappa B^2} = \mathcal{O}\left( \frac{\alpha}{T(1-\kappa B^2)} + \frac{\kappa G^2}{1-\kappa B^2}\right),$$
    which concludes the proof. 
\end{proof}


\newpage
\section{Description and theoretical guarantees of Byz-DASHA-PAGE}\label{app:dasha}
In this section, we discuss the interpretation of the theoretical results of \dasha~\cite{rammal2024communication} under the $(f,\kappa)-$robust aggregator. For being precise in our interpretation, we recall below the \dasha algorithm.

\begin{algorithm*}
    \caption{\dasha~\cite{rammal2024communication}}
    \textbf{Input:} Initial model $\theta^0 \in \R^d$ (chosen arbitrarily), initial update vector $R^0 = 0 \in \R^d$, total iterations $T \ge 1$, learning rate $\gamma$, robust aggregator $F$, momentum coefficient $\varrho \in [0, 1)$ 
    and, for each honest worker $w_i$: $g_i^0 = 0 \in \R^d$, $h_i^0 = 0 \in \R^d$, $m_i^0 = 0  \in \R^d$, and an unbiased compressor $\mathcal{Q}_i$ with compression parameter $\omega$.\\
  
    \vspace{5pt}
    \textbf{For} $t = 1$ to $T$:
    \begin{algorithmic}[1]
        \State \textbf{Server} sets
        \begin{align*}
            c^{t} = \left\{\begin{array}{cc}
                1, & \text{with probability } p  \\
                0, & \text{with probability } 1 - p
            \end{array} \right. .
        \end{align*}  
        \State \textbf{Server} broadcasts $R^{t-1}$ to all the workers. 
        \State For each \textbf{honest worker} $i$ do:
        \State \hspace{1em}Compute the updated model $\theta^t = \theta^{t-1} - \gamma R^t$.
        \State \hspace{1em}\textbf{If} $c^t = 1$ \textbf{then} 
        \State \hspace{1em}\hspace{1em} Compute $h_i^t = \lossworker{i}{t-1}$. 
        \State\hspace{1em}\textbf{else} 
        \State \hspace{1em}\hspace{1em} Compute $h_i^t = h_i^{t-1} + \widehat{ \Delta} (\theta^t, \theta^{t-1})$, \\
        \hspace{1em}\hspace{1em} where $\widehat{ \Delta} (\theta^t, \theta^{t-1})$ is an unbiased stochastic estimate of $\Delta (\theta^t, \theta^{t-1}) = \lossworker{i}{t} - \lossworker{i}{t-1}$. 
        \State \hspace{1em} Compute $ m_i^t = \mathcal{Q}_i\left( h_i^t -  h_i^{t-1} - \varrho (g_i^{t-1} - h_i^{t-1}) \right)$.
        \State \hspace{1em} Compute $g_i^t = g_i^{t-1} + m_i^t$.
        \State \hspace{1em} Send $m_i^t$ to the Server.


        \State \textbf{Server} computes $R^t = F(m_1^t, \ldots ,m_n^t)$ .

        
    \end{algorithmic} 
    \textbf{Output:} $\hat{\theta}_T$, chosen uniformly at random from $\{\theta^0, \ldots, \theta^{T-1}\}$.
    \label{alg:dasha}

\end{algorithm*}

\subsection{From $(f,\kappa)$-robust aggregation to $(\delta,c)$-robust aggregation}
Recall from Section \ref{sec:background} that our work uses $(f,\kappa)$-robust aggregation, defined as: An aggregation rule $F : (\mathbb{R}^d)^n \to \mathbb{R}^d$ is said to be $(f,\kappa)$-robust if, for any set of $n$ vectors $\{x_1,\ldots,x_n\} \subseteq \mathbb{R}^d$ and any subset $S \subseteq [n]$ of size $n-f$, it holds that
\[
    \norm{F(x_1,\ldots,x_n) - \overline{x}_S}^2 \leq \frac{\kappa}{|S|}\sum_{i \in S}\norm{x_i - \overline{x}_S}^2,
\]
where $\overline{x}_S \coloneqq \tfrac{1}{|S|} \sum_{i \in S} x_i$.

On the other hand, \dasha~\cite{rammal2024communication} uses a $(\delta,c)$-RAgg, defined via pairwise dispersion on the honest set $\mathcal H$ (with $|\mathcal H|\ge (1-\delta)n$):
\[ \textstyle 
\frac{1}{|\mathcal H|(|\mathcal H|-1)}\sum_{i\neq \ell\in \mathcal H}\mathbb{E}\|x_i-x_\ell\|^2 \;\le\; \sigma^2,
\qquad 
\mathbb{E}\|\hat x-\bar x\|^2 \;\le\; c\,\delta\,\sigma^2,
\]
where $\bar x=\tfrac1{|\mathcal H|}\sum_{i\in \mathcal H}x_i$.

Using the identity
\[
\frac{1}{|\mathcal H|}\sum_{i\in \mathcal H}\|x_i-\bar x\|^2=\frac{|\mathcal H|-1}{2|\mathcal H|}\cdot \frac{1}{|\mathcal H|(|\mathcal H|-1)}\sum_{i\neq \ell}\|x_i-x_\ell\|^2,
\]
we obtain $(f,\kappa)$-robust averaging implies $(\delta,c)$-RAgg with
\[ \textstyle
c\,\delta \;=\; \frac{\kappa}{2}\left(1-\frac{1}{|\mathcal H|}\right),
\]
where $|\mathcal H|=n-f $ and $ \delta=\frac{f}{n}$.
\subsection{Notation mapping from \cite{rammal2024communication}}
\begin{itemize}
  \item Honest set size: $|\mathcal H|=n-f$ and $\delta=f/n$.
  \item Heterogeneity: Byz-DASHA-PAGE’s $B$ and $\zeta^2$ correspond to our $B^2$ and $G^2$, respectively.
  \item Robustness constants: with $\tilde\kappa\coloneqq \kappa\left(1-\tfrac{1}{n-f}\right)$, we have
  $8c\delta=4\tilde\kappa$.
  \item Unbiased compression parameter $\omega$: for RandK, $\omega=\frac{d}{k}-1$. Writing $\alpha=d/k$, this is $\omega=\alpha-1$.
\end{itemize}
  
\subsection{Interpreting Theorem 2.2 in \cite{rammal2024communication} in our setup}

Theorem~2.2 in \cite{rammal2024communication} states that
for 
\begin{align}
  \gamma &\le \frac{1}{L+\sqrt{\eta}},\nonumber\\ 
  \eta& =\Big(8\omega(2\omega+1)(L_{\pm}^2+L^2)+\frac{1-p}{b}\big(12\omega(2\omega+1)+\tfrac2p\big)L_{\pm}^2\Big)\Big(\tfrac{1}{\sqrt{|H|}}+\sqrt{8c\delta}\Big)^{2},\nonumber\\
  \delta & < ((8c+4\sqrt c)\,B)^{-1} \nonumber
\end{align}
it holds
\[
\mathbb E\|\nabla f(\hat \theta)\|^2 \;\le\; \frac{1}{A}\left(\frac{2\delta_0}{\gamma T}+\Bigl(8c\delta+\sqrt{\tfrac{8c\delta}{|\mathcal H|}}\Bigr)\zeta^2\right),
\quad \textrm{ with }
A=1-\Bigl(8c\delta+\sqrt{\tfrac{8c\delta}{|\mathcal H|}}\Bigr)B.
\]

Rewriting the theorem in our notation, we obtain that for:
\begin{align}
& \gamma \le \frac{1}{L+\sqrt{\eta}}, \nonumber \\ 
& \eta =\Big(8\omega(2\omega+1)(L_{\pm}^2+L^2)+\frac{1-p}{b}\big(12\omega(2\omega+1)+\tfrac2p\big)L_{\pm}^2\Big)\Big(\tfrac{1}{\sqrt{|H|}}+\sqrt{8c\delta}\Big)^{2}, \nonumber \\
& \left(4\tilde\kappa+4\sqrt{\tfrac{\tilde\kappa\,f/n}{2}}\right)B^2< 1, \nonumber
\end{align}
it holds
\begin{align*}
& \mathbb E\|\nabla L_H(\hat\theta_T)\|^2
\;\le\;
\frac{1}{\,1-\bigl(4\tilde\kappa+2\sqrt{\tfrac{\tilde\kappa}{\,n-f\,}}\bigr)B^2\,}
\left(
\frac{2\delta_0}{\gamma T}
+\bigl(4\tilde\kappa+2\sqrt{\tfrac{\tilde\kappa}{\,n-f\,}}\bigr)G^2
\right).
\end{align*}


For RoSDHB, the theoretical bounds hold for full gradient descent, hence $p=1$. Moreover, we analyze our algorithm specifically for RandK unbiased compressor, hence $\omega=\frac{d}{k}-1 = \alpha - 1$.
Choosing the largest admissible stepsize $\gamma_{\max}=(L+\sqrt{\eta})^{-1}$ and considering the (favorable) case when $L_{\pm}=0$, we obtain that for
\begin{align*}
\left(4\tilde\kappa+4\sqrt{\tfrac{\tilde\kappa\,f/n}{2}}\right)B^2< 1,
\end{align*}
it holds that
\begin{align}
\label{eq:dasha_bound}
\mathbb E\|\nabla L_H(\hat\theta_T)\|^2
\;\le\; & 
\frac{1}{\,1-\bigl(4\tilde\kappa+2\sqrt{\tfrac{\tilde\kappa}{\,n-f\,}}\bigr)B^2\,}\nonumber\\
& \times \left[
\frac{2L\,\delta_0}{T}\Biggl(1+\sqrt{8(\alpha-1)(2\alpha-1)}\left(\tfrac{1}{\sqrt{n-f}}+2\sqrt{\tilde\kappa}\right)\Biggr)
+\bigl(4\tilde\kappa+2\sqrt{\tfrac{\tilde\kappa}{\,n-f\,}}\bigr)G^2
\right].
\end{align}

For $n-f\ge 2$, it holds that $\tilde \kappa \ge \kappa/2$.
As our goal is to compare our algorithm to Byz-DASHA-PAGE in the most favorable conditions to Byz-DASHA-PAGE, we consider  the less restrictive condition
\begin{align}
\kappa B^2< 1/2. \nonumber
\end{align}
Using the above in the RHS of~\eqref{eq:dasha_bound}, we obtain the following lower bound:
\begin{align}
\mathbb E\|\nabla L_H(\hat\theta_T)\|^2 &\leq \frac{1}{\,1-2\kappa B^2\,} \left[
\frac{2L\,\delta_0}{T}\Biggl(1+4(\alpha-1)\left(\frac{1}{\sqrt{2 n}}+\sqrt{\kappa}\right)\Biggr)
+\frac{\kappa}{2} G^2
\right] \nonumber\\
&\le  \mathcal O \left(
\frac{1+4(\alpha-1) \left(\frac{1}{\sqrt{n}}+\sqrt{\kappa}\right)}{(1-2\kappa B^2) T} + \frac{\kappa G^2}{1-2\kappa B^2}\right). \nonumber
\end{align}


\newpage 

\section{Additional Empirical Details \& Results}
\subsection{Byzantine attacks and defense}\label{app:attacks_defense}

To evaluate the resilience of our algorithm, we simulate two Byzantine attack strategies: \emph{Fall of Empires (FOE)}~\cite{xie2020fall} and \emph{A Little Is Enough (ALIE)}~\cite{baruch2019alittle}. Both of these attacks rely on crafting malicious updates that bias the aggregated updates while aiming to remaining undetected by robust defenses. We briefly describe them below.

Fix a time step $t$. Let $B^t$ denote the vectors sent by the byzantine workers in time step $t$. Recall from Section \ref{sec:background} that the byzantine workers can collude seamlessly, know the learning algorithm used by the server, and observe all messages exchanged between the server and the honest workers throughout training.  For $i \in \mathcal{H}$, suppose $m_i^t$ is the momentum of the honest worker $w_i$, and define $\eta \geq 0$ to be a fixed constant. Then, we have:

\noindent \textbf{Fall of Empires (FOE).} The byzantine updates are set as:
$$B^t = (1-\eta)\frac{1}{|\mathcal{H}|}\sum_{i\in\mathcal{H}} m_i^t.$$

\noindent \textbf{A Little Is Enough (ALIE).} Let $\sigma^t_\mathcal{H}$ be the coordinate-wise standard deviation of the momentum vectors recieved by the server from honest workers in time step $t$. Then, the byzantine updates are set as:
$$B^t = \frac{1}{|\mathcal{H}|}\sum_{i\in\mathcal{H}} m_i^t + \eta \sigma^t_\mathcal{H}$$

In our experiments for RoSDHB, we determine optimal value of $\eta$ that maximizes the L2-distance between the aggregated update $R^t$ and mean of the honest momentum vectors received by the server. Note that the index of coordinates sent by the byzantine workers are the same as those of the honest workers for our algorithm, since they need to adhere to the mask sent by the server in time step $t$. For \dasha, since the server does not decide the mask in every time step, workers have the flexibility to determine their mask independently. We incorporate this in the attack design for \dasha in the following manner: The byzantine workers independently sample the RandK mask in each time step, and modify only those masked coordinates for the attack strategy.

\noindent \textbf{Robust Aggregation.} In federated learning, the server aggregates the received updates to compute the global model in an iteration. In the presence of adversarial workers, the goal of the aggregator is to compute a function over the gradients such that the distance between the resulting global gradient and the gradient of the honest workers is minimized. We assume that $F$ is an $(f,\kappa)$-robust aggregator. In our experiments, we employ the Coordinate-Wise Trimmed Mean (CWTM) robust aggregation method from \cite{yin2018byzantine}.

\begin{definition}[Coordinate-Wise-Trimmed Mean]
    Fix a time step. Suppose $g_1,...,g_n \in \mathbb{R}^d$ are the gradient updates of the workers in the system, and $f$ is the number of byzantine workers in the system. For each coordinate $k \in [d]$, let $[\overline{g}_{1}]_k \le [\overline{g}_{2}]_k \le \cdots \le [\overline{g}_{n}]_k$ be the sorted values of $k^{th}$ coordinate of the $n$ gradient updates. Then, the coordinate-wise trimmed mean of $g_1,\dots ,g_n$, denoted by $CWTM(g_1,\dots g_n)$, is a vector in $\mathbb{R^d}$ such that the $k^{th}$ coordinate is given by:
    $$[CWTM(g_1,\dots g_n)]_k = \frac{1}{n-2f}\sum_{i = f+1}^{n-f}[\overline{g}_i]_k.$$
\end{definition}

\clearpage
\subsection{Experimental setup details}\label{app:experimental details}
\paragraph{System specification}
All experiments were conducted on a server equipped with four NVIDIA L40S GPUs (45 GiB each), two AMD EPYC 9124 16-core processors, and 512 GiB of RAM. The software environment consisted of Python 3.10.6 and PyTorch 2.6.0, running with CUDA 12.4 and cuDNN 9.1 on Debian GNU/Linux 11 (Bullseye).

\paragraph{Hyperparameters}
\renewcommand{\arraystretch}{1.10}
\begin{table*}[!htbp]
    \centering
    \begin{tabular}{|p{3.5cm}|*{2}{c|}c|c|c|c|c|c|c|c|}
        \hline
        Compression & \multicolumn{5}{c|}{MNIST} & \multicolumn{4}{c|}{CIFAR-10} \\
        \cline{2-10}
         Ratio, $k/d$ & 0.05 & 0.10 & 0.30 & 0.50 & 1.00 & 0.25 & 0.5 & 0.75 & 1.00 \\
        \hline
        Learning Rate $\gamma$& 0.16 & 0.16 & 0.4 & 0.4 & 0.4 & 0.7 & 0.7 & 0.7 & 0.7    \\
        \hline
        Grid Search for $\gamma$  &  \multicolumn{5}{c|}{ \{0.008,0.04,0.08, 0.16, 0.4\} }  &  \multicolumn{4}{c|}{\{0.1,0.2,0.3,0.5,0.7\}}  \\
        \hline
        Momentum, $\beta$  &  \multicolumn{5}{c|}{0.8}  &  \multicolumn{4}{c|}{0.8}  \\
        \hline
        Batch Size &  \multicolumn{5}{c|}{full batch 6000}  &  \multicolumn{4}{c|}{128}  \\
        \hline
        \# Byzantine Workers &  \multicolumn{5}{c|}{$f = \{1,3,5 \}$}  &  \multicolumn{4}{c|}{$f = \{1,3,5 \}$}  \\
        \hline
    \end{tabular}
    \caption{Summary of hyperparameters used for \textsc{RoSDHB} experiments when data partition follows a symmetric Dirichlet distribution with parameters $w = 5$ . }
    \label{tab:hyper_param_rosdhb}
\end{table*}

\renewcommand{\arraystretch}{1.10}
\begin{table*}[hbt!]
    \centering
    \begin{tabular}{|p{3.5cm}|*{2}{c|}c|c|c|c|c|c|c|c|}
        \hline
        Compression & \multicolumn{5}{c|}{MNIST} & \multicolumn{4}{c|}{CIFAR-10} \\
        \cline{2-10}
         Ratio, $k/d$ & 0.05 & 0.10 & 0.30 & 0.50 & 1.00 & 0.25 & 0.5 & 0.75 & 1.00 \\
        \hline
        Learning Rate $\gamma$& 0.008 & 0.008 & 0.008 & 0.08 & 0.16 & 0.03 & 0.03 & 0.03 & 0.03    \\
        \hline
        Grid Search for $\gamma$  &  \multicolumn{5}{c|}{ \{0.004, 0.0056, 0.008, 0.04, 0.08, 0.16, 0.4\} }  &  \multicolumn{4}{c|}{\{ 0.01,0.03,0.05, 0.1,0.2\}}  \\
        \hline
        Momentum, $\varrho$ &  \multicolumn{5}{c|}{$1/(2(d/k)-1)$}  &  \multicolumn{4}{c|}{$1/(2(d/k)-1)$}  \\
        \hline
        Batch Size &  \multicolumn{5}{c|}{full batch 6000}  &  \multicolumn{4}{c|}{128}  \\
        \hline
        \# Byzantine Workers &  \multicolumn{5}{c|}{$f = \{1,3,5\}$}  &  \multicolumn{4}{c|}{$f = \{1,3,5 \}$}  \\
        \hline
    \end{tabular}
    \caption{Summary of hyper-parameters used for \textsc{Byz-DASHA-PAGE} (described in Algorithm~\ref{alg:dasha}) experiments when data partition follows a symmetric Dirichlet distribution with parameters $w = 5$.\protect\footnotemark}
    \label{tab:hyper_param_dasha}
\end{table*}

\footnotetext{We chose the optimal momentum $\alpha$ for Byz-DASHA-PAGE as suggested in their paper.}

\begin{table*}[!htbp]
    \centering
    \begin{tabular}{|p{3.5cm}|*{2}{c|}c|c|c|c|c|c|c|}
        \hline
        Heterogeneity & \multicolumn{4}{c|}{\textsc{RoSDHB}} & \multicolumn{4}{c|}{\textsc{Byz-DASHA-PAGE}} \\
        \cline{2-9}
         parameter, $w$ & 0.5 & 1.0 & 3.0 & 5.0 & 0.5 & 1.0 & 3.0 & 5.0  \\
        \hline
        Learning Rate $\gamma$& 0.4 & 0.4 & 0.4 & 0.16 & 0.04 & 0.008 & 0.008 & 0.008    \\
        \hline
        Grid Search for $\gamma$  &  \multicolumn{4}{c|}{ \{0.008,0.04,0.08, 0.16, 0.4\} }  &  \multicolumn{4}{c|}{\{0.008,0.04,0.08, 0.16, 0.4\}}  \\
        \hline
        Momentum & \multicolumn{4}{c|}{$\beta =$ 0.8}  &  \multicolumn{4}{c|}{$\varrho =$ 1/19}  \\
        \hline
    \end{tabular}
    \caption*{Table: Summary of hyperparameters used for \textsc{RoSDHB} and \textsc{Byz-DASHA-PAGE} experiments on MNIST dataset with $k/d=0.1$, $f=1$ and full batch under different data partitions follows a symmetric Dirichlet distribution with parameters~$w$ (Experiment in Figure~\ref{fig:data_hetero_effect}). }
\end{table*}

\newpage

\subsection{Further interpretations of the main results}
\label{app:further_results}



\begin{table}[h]
\centering
\scriptsize
\setlength{\tabcolsep}{5pt}
\begin{tabular}{|c|c|c|c|c|c|c|}
\hline
\textbf{Accuracy} & \textbf{RoSDHB (f=1)} & \textbf{SOTA (f=1)} & \textbf{Speedup (f=1)} & \textbf{RoSDHB (f=3)} & \textbf{SOTA (f=3)} & \textbf{Speedup (f=3)} \\
\hline
0.76 & 10.8 ± 1.04 & 55.8 ± 7.87 & \textbf{5.17$\times$} & 17.0 ± 1.17 & 93.4 ± 11.93 & \textbf{5.49$\times$} \\
0.78 & 12.0 ± 1.02 & 61.8 ± 8.32 & \textbf{5.15$\times$} & 19.0 ± 1.17 & 104.4 ± 13.11 & \textbf{5.49$\times$} \\
0.80 & 13.0 ± 1.23 & 70.2 ± 8.89 & \textbf{5.40$\times$} & 22.8 ± 1.18 & 117.0 ± 14.60 & \textbf{5.13$\times$} \\
0.82 & 15.6 ± 1.37 & 82.8 ± 9.97 & \textbf{5.31$\times$} & 25.6 ± 1.25 & 133.6 ± 16.52 & \textbf{5.22$\times$} \\
0.84 & 18.8 ± 1.15 & 102.6 ± 10.93 & \textbf{5.46$\times$} & 31.2 ± 1.91 & 162.2 ± 19.79 & \textbf{5.20$\times$} \\
0.86 & 23.0 ± 1.10 & 135.8 ± 12.20 & \textbf{5.90$\times$} & 37.6 ± 2.54 & 168.0 ± 10.03 & \textbf{4.47$\times$} \\
0.88 & 32.0 ± 1.36 & 201.3 ± 7.49 & \textbf{6.29$\times$} & 53.0 ± 3.09 & -- & -- \\
0.90 & 56.8 ± 2.60 & -- & -- & 90.8 ± 5.06 & -- & -- \\
\hline
\end{tabular}
\caption{Convergence time (mean ± std) required to reach various test threshold accuracy on MNIST with compression ratio \(k/d = 0.1\), comparing \textsc{RoSDHB} and Byz-DASHA-PAGE (SOTA) under $f = 1,3$ Byzantine workers . \textsc{RoSDHB} consistently outperforms SOTA, achieving up to \(6.3\times\) speedup.}
\label{tab:mnist_speedup_table}
\end{table}

\begin{table}[h!]
\centering
\scriptsize
\setlength{\tabcolsep}{5pt}
\begin{tabular}{|c|c|c|c|c|c|c|}
\hline
\textbf{Accuracy} & \textbf{RoSDHB (f=1)} & \textbf{SOTA (f=1)} & \textbf{Speedup (f=1)} & \textbf{RoSDHB (f=3)} & \textbf{SOTA (f=3)} & \textbf{Speedup (f=3)} \\
\hline
0.30 & 13.2 ± 0.36 & 54.4 ± 2.11 & \textbf{4.12$\times$} & 26.4 ± 0.18 & 87.6 ± 3.17 & \textbf{3.32$\times$} \\
0.32 & 15.6 ± 0.33 & 66.0 ± 2.42 & \textbf{4.23$\times$} & 28.8 ± 0.78 & 98.0 ± 3.66 & \textbf{3.40$\times$} \\
0.34 & 17.2 ± 0.22 & 78.0 ± 2.37 & \textbf{4.53$\times$} & 36.8 ± 0.61 & 108.0 ± 3.58 & \textbf{2.93$\times$} \\
0.36 & 19.6 ± 0.44 & 90.0 ± 3.16 & \textbf{4.59$\times$} & 43.6 ± 1.21 & 125.6 ± 3.00 & \textbf{2.88$\times$} \\
0.38 & 23.6 ± 0.33 & 106.4 ± 3.65 & \textbf{4.51$\times$} & 55.2 ± 1.99 & 151.6 ± 4.28 & \textbf{2.75$\times$} \\
0.40 & 26.0 ± 0.69 & 139.2 ± 5.79 & \textbf{5.35$\times$} & 80.4 ± 3.12 & 182.8 ± 2.48 & \textbf{2.27$\times$} \\
0.42 & 29.6 ± 0.59 & 164.8 ± 5.70 & \textbf{5.57$\times$} & 98.4 ± 2.30 & 248.0 ± 6.34 & \textbf{2.52$\times$} \\
0.44 & 33.6 ± 0.44 & 210.5 ± 8.71 & \textbf{6.26$\times$} & 111.6 ± 3.18 & -- & -- \\
\hline
\end{tabular}
\caption{Convergence time (mean ± std in epochs) to reach various test threshold accuracy  on CIFAR-10 for compression ratio $k/d = 0.5$, comparing \textsc{RoSDHB} and Byz-DASHA-PAGE (SOTA) under different numbers of Byzantine workers \(f=1\) and \(f=3\).}
\label{tab:cifar_speedup_table}
\end{table}

\begin{table}[h!]
\centering
\scriptsize
\setlength{\tabcolsep}{5pt}
\begin{tabular}{|c|c|cc|cc|c|cc|cc|c|}
\hline
\textbf{f} & \textbf{k/d} & \multicolumn{2}{c|}{\textbf{SOTA Time}} & \multicolumn{2}{c|}{\textbf{RoSDHB Time}} & \textbf{Speedup} & \multicolumn{2}{c|}{\textbf{SOTA Coords}} & \multicolumn{2}{c|}{\textbf{RoSDHB Coords}} & \textbf{Comm.} \\
           &               & Mean  & Std      & Mean  & Std        & $\scriptstyle{\frac{\text{SOTA}}{\text{RoSDHB}}}$ & Mean      & Std            & Mean     & Std          & \textbf{Savings (\%)} \\
\hline
1 & 0.05  & 161.0   & 21.81   & 81.0    & 8.84    & 1.99× & 95231.5   & 12902.84  & 47911.5  & 5229.34  & 49.67 \\
1 & 0.1   & 117.0   & 11.44   & 20.6    & 1.28    & 5.68× & 138411.0  & 13537.99  & 24369.8  & 1518.67  & 82.38 \\
1 & 0.3   & 109.2   & 11.65   & 12.0    & 0.40    & 9.10× & 387550.8  & 41332.06  & 42588.0  & 1419.60  & 89.01 \\
1 & 0.5   & 95.25   & 30.41   & 10.8    & 0.72    & 8.82× & 563403.8  & 179865.79 & 63882.0  & 4232.43  & 88.65 \\
1 & 1.0   & 19.6    & 2.63    & 10.6    & 0.73    & 1.85× & 231868.0  & 31137.85  & 125398.0 & 8596.10  & 45.93 \\
\hline
3 & 0.05  & 193.0   & 14.35   & 133.8   & 14.91   & 1.44× & 114159.5  & 8486.32   & 79142.7  & 8820.14  & 30.64 \\
3 & 0.1   & 185.4   & 23.34   & 33.4    & 1.54    & 5.55× & 219328.2  & 27609.51  & 39512.2  & 1820.44  & 82.00 \\
3 & 0.3   & 191.5   & 13.81   & 21.0    & 1.52    & 9.12× & 679633.5  & 49024.08  & 74529.0  & 5405.68  & 89.01 \\
3 & 0.5   & 117.0   & 34.34   & 18.4    & 1.43    & 6.36× & 692055.0  & 203148.65 & 108836.0 & 8464.86  & 84.28 \\
3 & 1.0   & 34.8    & 4.88    & 18.2    & 1.11    & 1.91× & 411684.0  & 57703.24  & 215306.0 & 13130.77 & 47.70 \\
\hline
\end{tabular}
\caption{Convergence time (mean ± std) and number of coordinates communicated (mean ± std) to achieve a threshold accuracy of $85\%$ for Byz-DASHA-PAGE (SOTA) and RoSDHB on MNIST dataset under various compression ratios ($k/d$) and Byzantine settings ($f$). Speedup indicates how many times faster RoSDHB converges compared to SOTA. Communication savings shows percentage reduction in communicated coordinates by RoSDHB compared to SOTA computed as $\text{Communication Savings (\%)} = \left( \frac{\text{SOTA Mean} - \text{RoSDHB Mean}}{\text{SOTA Mean}} \right) \times 100$ .}
\label{tab:convergence_time_coords_speedup_savings}
\end{table}

\begin{table}[h!]
\centering
\scriptsize
\setlength{\tabcolsep}{5pt}
\begin{tabular}{|c|c|cc|cc|c|cc|cc|c|}
\hline
\textbf{f} & \textbf{k/d} & \multicolumn{2}{c|}{\textbf{SOTA Time}} & \multicolumn{2}{c|}{\textbf{RoSDHB Time}} & \textbf{Speedup} & \multicolumn{2}{c|}{\textbf{SOTA Coords}} & \multicolumn{2}{c|}{\textbf{RoSDHB Coords}} & \textbf{Comm. } \\
           &              & Mean & Std & Mean & Std & $\scriptstyle{\frac{\text{SOTA}}{\text{RoSDHB}}}$ & Mean & Std & Mean & Std & \textbf{Savings (\%)} \\
\hline
1 & 0.25 & 127.2 & 7.28 & 25.2 & 0.44 & 5.05× & 382.6M & 21.90M & 70.4M & 1.22M & 81.59 \\
1 & 0.5  & 139.2 & 11.57 & 26.0 & 1.39 & 5.35× & 797.6M & 66.32M & 145.3M & 7.74M & 81.79 \\
1 & 0.75 & 150.0 & 8.52 & 28.4 & 1.31 & 5.28× & 1267.8M & 72.04M & 238.0M & 11.02M & 81.22 \\
\hline
3 & 0.25 & 151.6 & 10.90 & 86.8 & 3.47 & 1.75× & 456.0M & 32.79M & 242.5M & 9.69M & 46.84 \\
3 & 0.5  & 182.8 & 4.95  & 80.4 & 6.23 & 2.27× & 1047.4M & 28.37M & 449.2M & 34.82M & 57.12 \\
3 & 0.75 & 246.0 & 11.11 & 67.6 & 4.94 & 3.64× & 2079.2M & 93.93M & 566.5M & 41.44M & 72.76 \\
\hline
\end{tabular}
\caption{Convergence time (mean ± std) and number of coordinates communicated (mean ± std) to achieve a threshold accuracy of $40\%$ for Byz-DASHA-PAGE (SOTA) and RoSDHB on CIFAR-10 under different compression ratios (\(k/d\)) and Byzantine worker settings (\(f\)). }
\label{tab:cifar10_comm_time_bits_savings}
\end{table}

\clearpage
\subsection{Full set of experiments}\label{app:fullset_experiments}
We present our experimental results where the system is under the FOE attack and the ALIE attack for the MNIST and CIFAR-10 datasets. 

We model a system consisting of 10 honest workers. We consider three Byzantine regimes: $f = 1$, $f = 3$, and $f = 5$. We evaluate the performance of our algorithm under the FOE and ALIE attacks and compare against SOTA approach - \dasha from \cite{rammal2024communication} for the CWTM aggregation rule (detailed in Appendix \ref{app:attacks_defense}). The plots are presented below and complement our (partial) results in Section \ref{sec:experiments} of the main
paper. 

\paragraph{Comprehensive results for FOE attack}\label{sec:compFOE}

\begin{figure}[h!]
    \centering
    \includegraphics[scale=0.375]{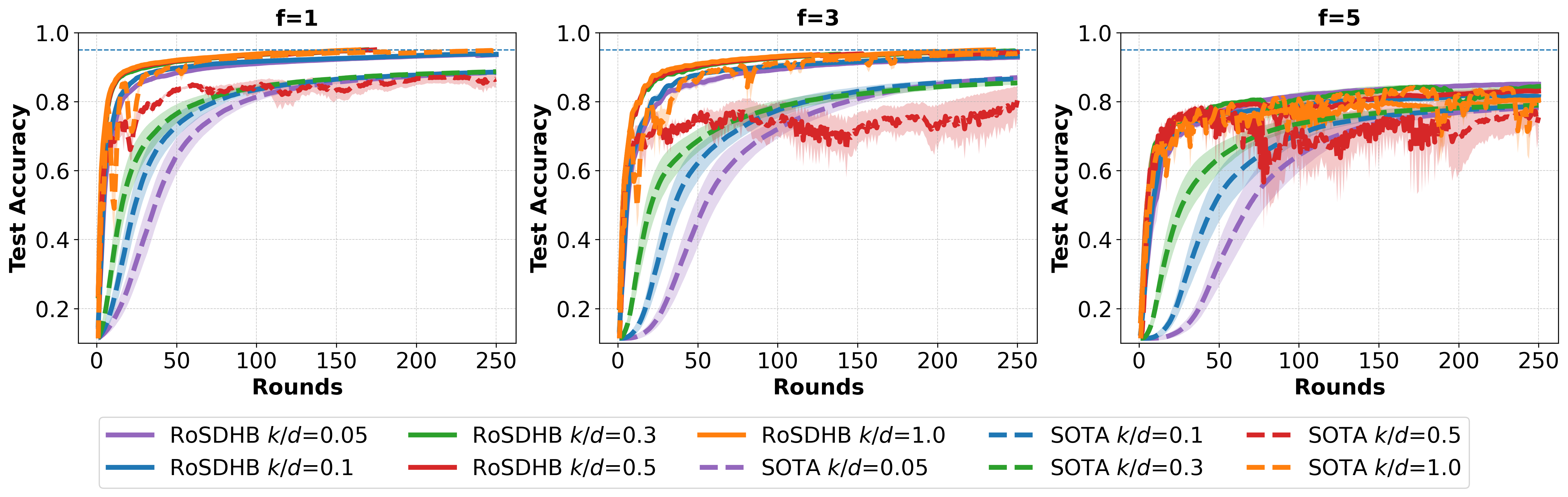}
    \caption{The convergence plots of RoSDHB and \dasha (SOTA) on MNIST under the FOE attack for varying sampling ratio $k/d \in \{ 0.05, 0.1, 0.3, 0.5, 1.0\}$ and number of byzantine workers $f \in \{1,3,5\}$.}
    \label{fig:MNIST_foe}
\end{figure}

\begin{figure}[h!]
    \centering
    \includegraphics[scale=0.375]{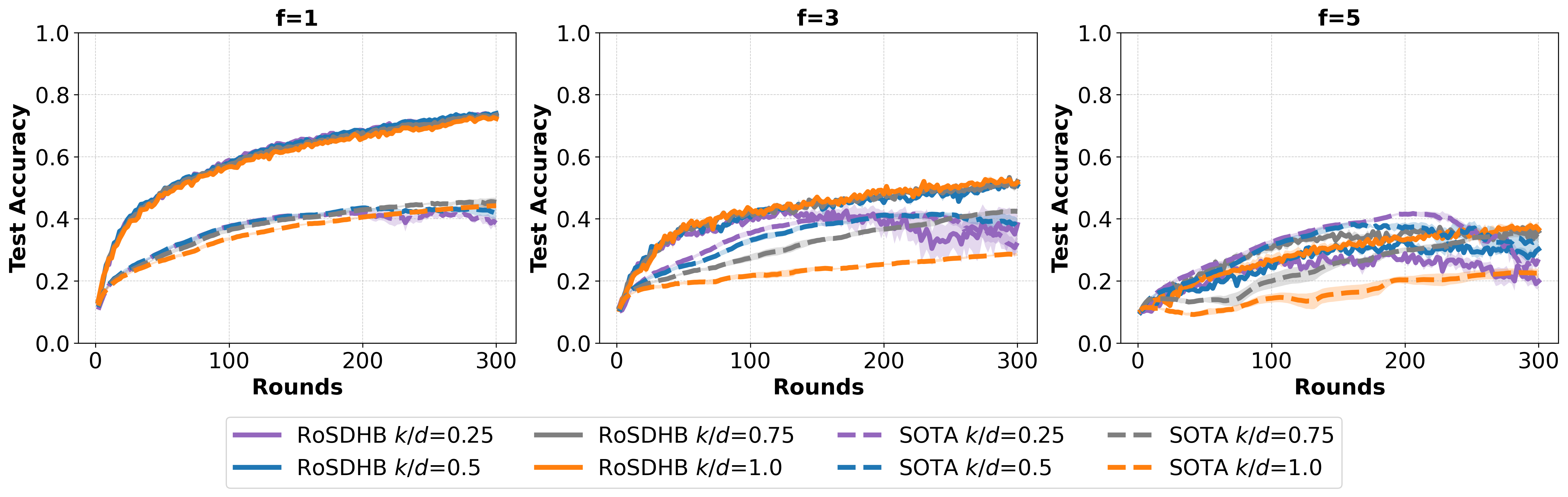}
    \caption{The convergence plots of stochastic version of RoSDHB and \dasha (SOTA) on CIFAR-10 under the FOE attack for varying sampling ratio $k/d \in \{0.25,0.5,0.75,1.0 \}$ and number of byzantine workers $f \in \{1,3,5\}$.}
    \label{fig:Cifar_foe}
\end{figure}

\newpage
\paragraph{Comprehensive results for ALIE attack}\label{sec:compALIE}
\begin{figure}[h!]
    \centering
    \includegraphics[scale=0.375]{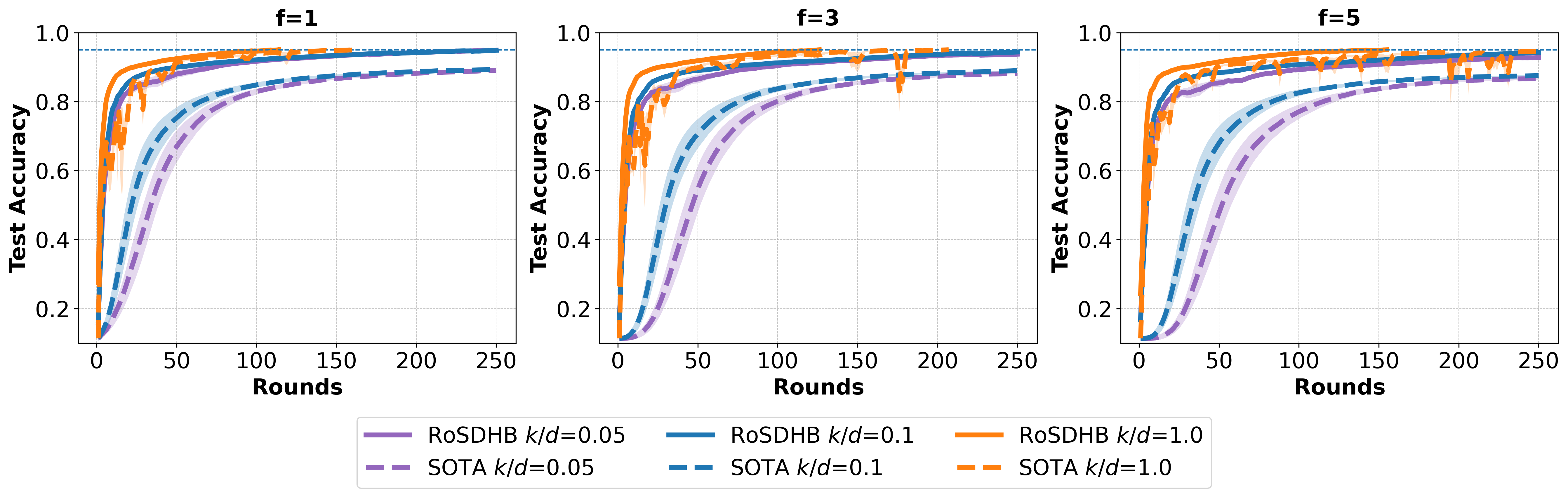}
    \caption{The convergence plots of RoSDHB and \dasha (SOTA) on MNIST under the ALIE attack for varying sampling ratio $k/d \in \{0.05,0.1,1.0\}$ and number of byzantine workers $f \in \{ 1,3,5 \} $.}
    \label{fig:MNIST_alie}
\end{figure}
\begin{figure}[h!]
    \centering
    \includegraphics[scale=0.375]{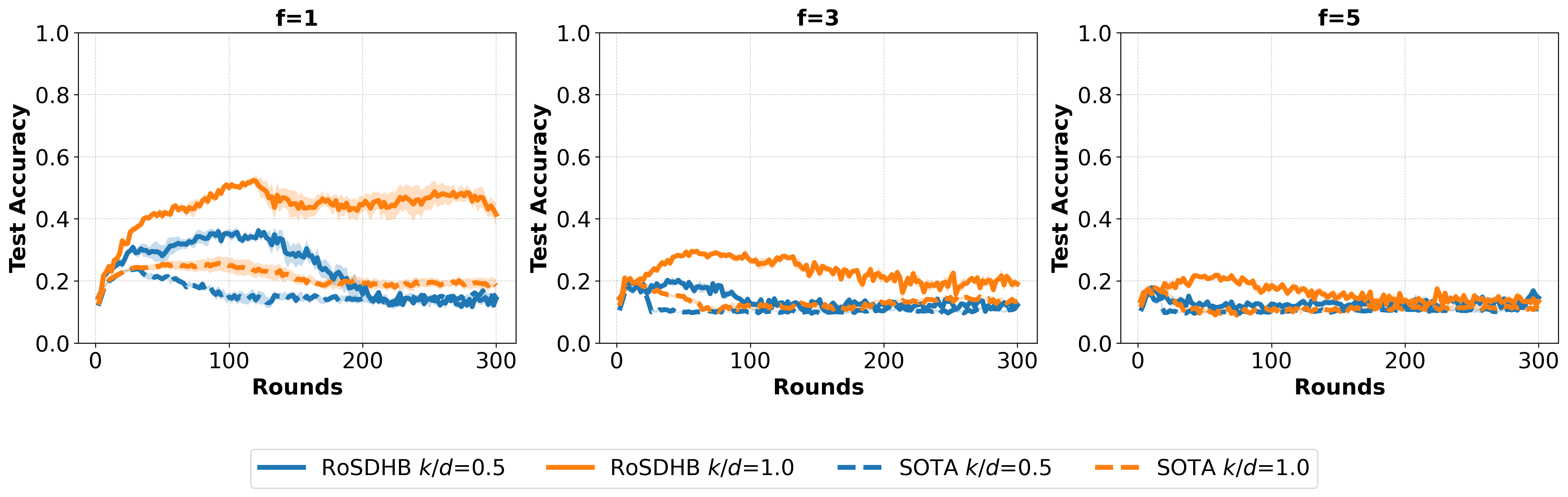}
    \caption{The convergence plots of stochastic version of RoSDHB and \dasha (SOTA) on CIFAR-10 under the ALIE attack for varying sampling ratio $k/d \in \{0.5,1.0\}$ and number of byzantine workers $f \in \{ 1,3,5 \} $.}
    \label{fig:cifar10_alie}
\end{figure}

\end{document}